\documentclass[runningheads]{llncs}

\usepackage{eccv}

\usepackage{eccvabbrv}
\usepackage{placeins}
\usepackage{graphicx}
\usepackage{booktabs}

\usepackage[accsupp]{axessibility}  

\usepackage{wrapfig}     
\usepackage{caption}     
\usepackage{subcaption}
\usepackage{graphicx}    
\usepackage{array}       
\usepackage{booktabs}    
\usepackage{color}

\usepackage{algorithm}   
\usepackage{algpseudocode} 
\usepackage{amsfonts}    
\usepackage{amssymb}     
\usepackage{enumitem}    
\usepackage{multirow}    
\usepackage{mathtools}   

\usepackage{url}            
\usepackage{nicefrac}       
\usepackage{microtype}      
\usepackage{xcolor}         
\usepackage{amsmath}
\usepackage{float}       
\usepackage{comment}
\usepackage{subcaption}
\usepackage{amsmath}
\usepackage{varwidth}
\usepackage{amsmath}
\usepackage{amsfonts}

\usepackage{rotating}   

\usepackage{textcomp}   
\usepackage[table]{xcolor} 

\definecolor{mygreen}{rgb}{0.13, 0.55, 0.13}
\usepackage{colortbl}
\usepackage{subcaption}
\usepackage{float}

\usepackage{hyperref}

\usepackage{orcidlink}

\begin{document}

\title{SAMPLe: A Sharpness Aware Minimization based Optimizer for Prompt Learning in Vision-Language Models} 

\titlerunning{SAMPLe: SAM-based Optimizer for Prompt Learning in VLMs}

\author{Hossein Rajoli\inst{1,2}\orcidlink{0009-0000-1847-400x}$^{\star}$ \and
Fatemeh Lotfi\inst{1}\orcidlink{0009-0009-1691-0029}$^{\star}$ \and
Niloufar Alipour\inst{1}\orcidlink{0009-0000-6881-3671} \and
Hossein Kashiani\inst{1}\orcidlink{0000-0001-8338-9987} \and
Xiaolong Ma\inst{3}\orcidlink{0000-0003-3753-7648} \and
Fatemeh Afghah\inst{1}\orcidlink{0000-0002-2315-1173}}

\authorrunning{H.~Rajoli et al.}

\institute{Clemson University, Clemson SC 29634, USA \\
\email{\{hrajoli,flotfi, nalipou, hkashia, fafghah\}@clemson.edu} \and
Siemens Energy (AI Lab), Orlando FL 32826, USA\\
\email{hossein.rajoli.nowdeh@siemens-energy.com} \and
University of Arizona, Tucson AZ 85721, USA \\
\email{xiaolongma@arizona.edu}}

\maketitle
\renewcommand\thefootnote{$\star$}
\footnotetext{Hossein Rajoli and Fatemeh Lotfi contributed equally to this work.}
\renewcommand\thefootnote{\arabic{footnote}}

\begin{abstract}
Pre-trained Vision-Language Models (VLMs) like CLIP have proven highly effective as foundation models for various downstream applications. However, prompt learning in VLMs encounters a performance-generalization dilemma: while prompts can be tuned to achieve high accuracy on seen distributions, this tuning process often undermines their generalizability to unseen data. The limited set of learnable prompts, which contextualize and condition the input to steer it toward the task within the pretrained VLM, tends to overfit the training data, leading to a trade-off between task-specific performance and preserving generalization. To address this dilemma, we introduce SAMPLe (Sharpness-Aware Minimization Prompt Learning), a plug-in sharpness-aware optimizer that enhances prompt generalizability by accounting for loss landscape sharpness. Unlike conventional methods, SAMPLe balances exploration and exploitation by satisfying objective function constraints at each step, dynamically adapting to the current optimization state based on the local curvature and gradient properties. This approach reduces overfitting on seen distributions and improves adaptability to unseen data, preserving the generalization potential of pre-trained VLM models. We integrate SAMPLe into multiple prompt learning frameworks, including CoOp, CoCoOp, MaPLe, TCP, and Co-Prompt, demonstrating its effectiveness across diverse methods. 
Experiments show that SAMPLe elevates prompt learning frameworks and consistently outperforms existing optimizers across diverse settings, establishing itself as a robust, model-agnostic solution for prompt learning.\vspace{-0.4cm}
\end{abstract}    
\section{Introduction}
\label{sec:intro}

\begin{figure}[t]
    \centering
    \vspace{0cm}

    \begin{minipage}{0.22\textwidth}
        \includegraphics[width=\linewidth]{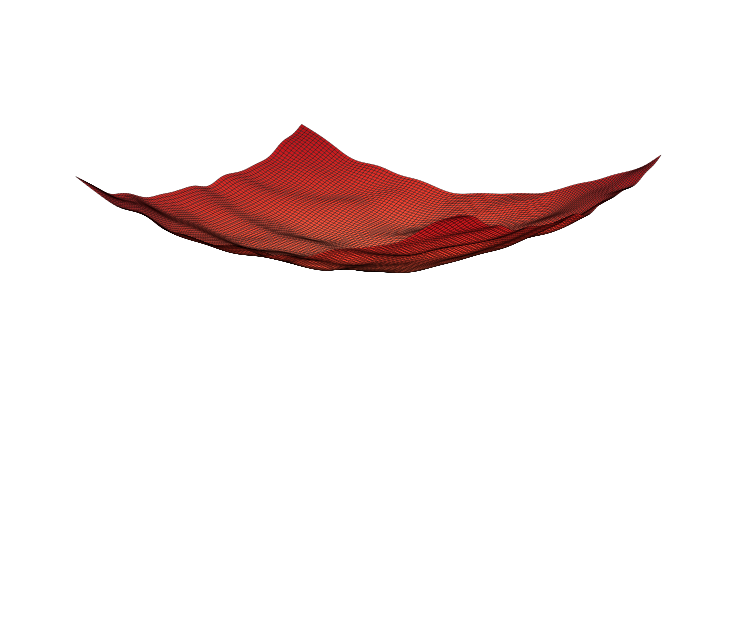}\vspace{-0.9cm}\\
        \centering \scriptsize (a) FSAM
    \end{minipage}
    \hspace{0.02\textwidth}
    \begin{minipage}{0.22\textwidth}
        \includegraphics[width=\linewidth]{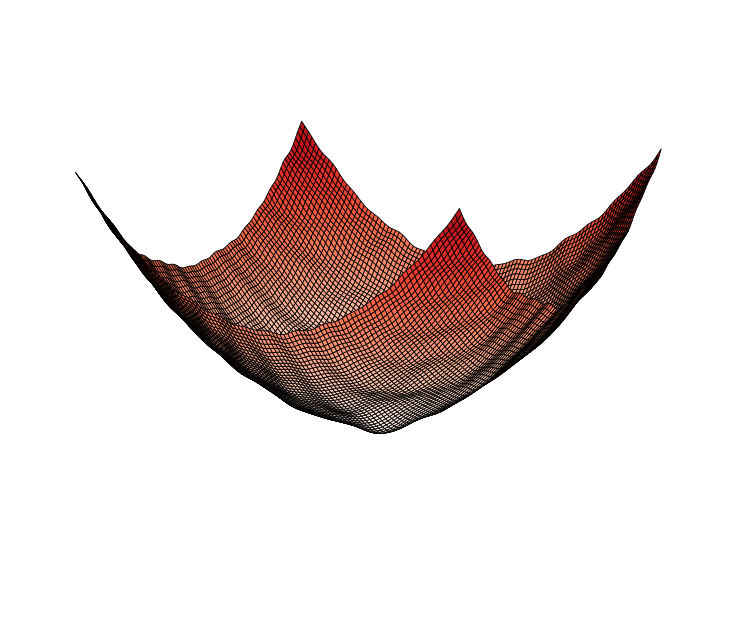} \vspace{-0.8cm}\\
        \centering \scriptsize (b) SAGM
    \end{minipage}
    \hspace{0.02\textwidth}
    \begin{minipage}{0.22\textwidth}
        \includegraphics[width=\linewidth]{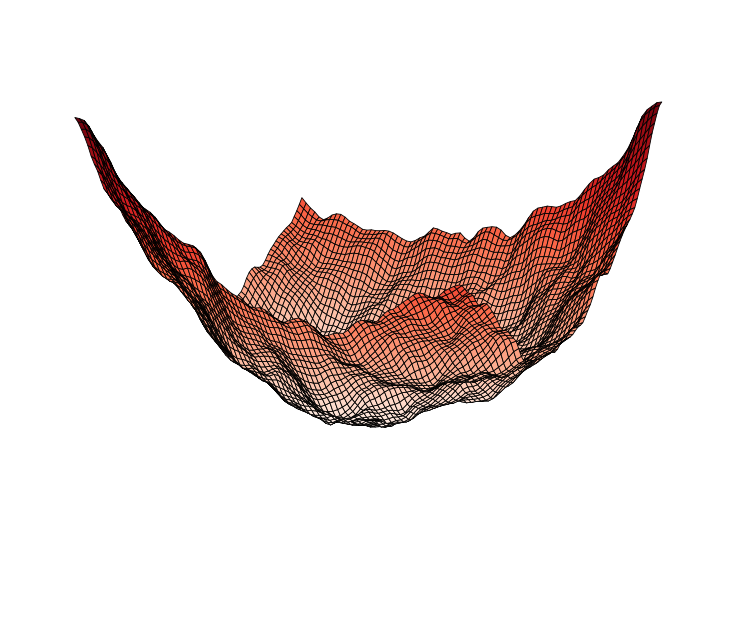}\vspace{-0.8cm}\\
        \centering \scriptsize (c) SAMPLe
    \end{minipage}

    \caption{\small Normalized loss landscapes of CoOp\cite{zhou2022conditional} on ImageNet using F-SAM\cite{li2024friendly}, SAGM\cite{wang2023sharpness}, and the proposed SAMPLe, each scaled by the maximum absolute loss (among FSAM, SAGM, and SAMPLe) preserving relative depth and sharpness. The visualization demonstrates SAMPLe’s effectiveness in achieving both flatter minima and lower empirical risk.}
    \label{fig:loss_landscape}
    \vspace{-10pt}
\end{figure}

Vision-Language Models (VLMs) such as CLIP have become essential for a wide range of visual tasks due to their robust multimodal understanding, leveraging both vision and text to achieve notable zero-shot performance across domains~\cite{radford2021learning,jia2021scaling}. For downstream applications, prompt learning has emerged as an efficient alternative to traditional fine-tuning by introducing learnable prompts that adapt VLMs to specific tasks while keeping their parameters fixed~\cite{zhou2022learning,zhou2022conditional}. However, a core challenge persists in balancing high task-specific accuracy while retaining generalization to unseen classes~\cite{dziugaite2017computing,keskar2016large} under the constraint of a limited number of learnable parameters.

This restriction amplifies sensitivity in optimization, since updates are confined to a narrow parameter subspace of the learnable prompts. Authors in ~\cite{novak2018sensitivity} showed that smaller capacity models often converge to solutions that are less robust to perturbations, leading to weaker generalization. On the other hand, \cite{keskar2016large} further linked such generalization gaps to sharp minima characterized by large positive eigenvalues of the Hessian, while \cite{ishida2020need} emphasized that even with low training error, a large train–test gap persists under sharp solutions. These findings highlight that the limited learnable space in prompt learning makes models especially vulnerable to sharp, high-curvature minima in the loss landscape.
Although various techniques \cite{yao2023visual, zhou2022conditional, zhou2022learning, khattak2023maple, zhu2023prompt} have been proposed to preserve the generalizability of VLMs while adapting them to downstream tasks, just \cite{10814656} considered role of gradient manipulation with respect to the loss landscape. 

Incorporating model-agnostic loss landscape awareness into VLM prompt learning can substantially enhance generalization. By guiding optimization toward flatter and more stable minima, the model can produce prompts that are both robust and transferable across tasks. However, because learnable prompts involve only a small number of tunable parameters, they are especially prone to converging to sharp, non-robust minima. While flatter landscapes are crucial for generalization, prior work shows they may come at the cost of reduced accuracy on seen distributions~\cite{wang2023sharpness,zhuang2022surrogate}. Thus, achieving effective prompts requires carefully balancing the minimization of training loss (exploitation) with the search for sufficiently flat regions of the landscape (exploration), ensuring strong performance on both seen and unseen classes. 

To address these challenges, we propose SAMPLe (Sharpness-Aware Minimization Prompt Learning), a framework specifically designed to incorporate sharpness-aware minimization (SAM) into prompt learning. SAMPLe directly targets the trade-off between minimizing loss values and maintaining a flat loss landscape, ensuring both accuracy and generalization. It introduces three essential strategies to achieve this balance. First, SAMPLe optimizes for a loss minimum that is sufficiently low in training loss and flat in the loss landscape, providing both stability and robustness. Second, it ensures that SAM gradient updates, computed at the deviated point (see Sec.\ref{sec:SAMPLe}), are coherently aligned with empirical risk minimization (ERM) gradients, as illustrated in Fig.\ref{fig:sample_comparison}a and Fig.\ref{fig:sample_comparison}b. This alignment addresses the instability caused by the limited parameter space of prompts. Finally, SAMPLe constrains the SAM gradient updates at the perturbed point to be orthogonal to the full-batch gradient (refer to Sec.\ref{sec:full-batch}, Sec.\ref{sec:SAMPLe}), effectively aligning it with a relaxed version of the ERM gradient that seeks the flattest minima. Satisfying both constraints of alignment to the ERM gradient and its relaxed version in each iteration is the central objective of the SAMPLe optimizer, ensuring a precise adaptive balance between exploration and exploitation.

These dual objectives are critical for effective prompt learning. However, sharpness-aware minimization optimizers, including SAM, F-SAM\cite{li2024friendly}, and SAGM\cite{wang2023sharpness} (see Fig.\ref{fig:sample_comparison}a and Fig.\ref{fig:comparison}), fail to balance both conditions. Beyond these optimizers, a recently published study \cite{10814656} proposed GCSCoOP, a SAM-based method designed to address the generalization problem by considering both loss value and loss sharpness. However, it suffers from limited scalability, as it relies on a heuristic gradient computation with two hyperparameters ($\beta_1$, $\beta_2$). In contrast, SAMPLe updates gradients by automatically balancing exploitation and exploration constraints considering the current state of optimization. 

The key contributions of this work can be summarized as follows:
\begin{itemize}
    \item  We introduce SAMPLe, a SAM-based optimizer specifically designed for VLM prompt learning. It comprehensively defines a dual-objective optimization framework that adaptively balances exploiting low ERM loss and exploring the flattest minima at every iteration.
    
    \item We analyze the limitations of existing sharpness-aware optimization methods in VLM prompt learning and establish the necessity of an adaptive dual-objective approach, motivating the design of SAMPLe.

    \item We demonstrate the superior performance of SAMPLe over SOTA in VLM prompt learning and provide an ablation study comparing it with related sharpness-aware optimizers.
  
\end{itemize}

\begin{figure*}[h!]
    \centering
    
    \begin{minipage}{0.24\textwidth}
        \centering
        \includegraphics[width=\linewidth]{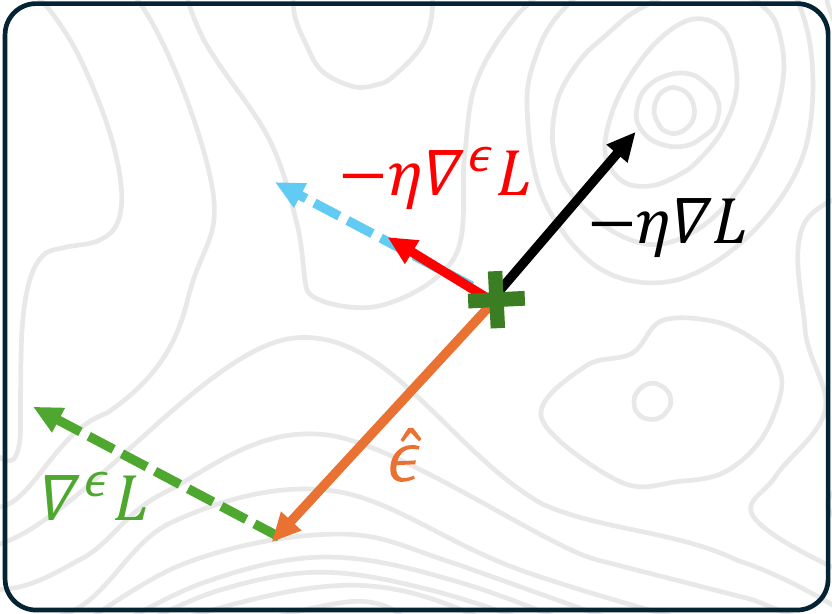}
        \caption*{\fontsize{7pt}{7pt}\selectfont{{(a) SAM across all stages of optimization.}}} 
    \end{minipage}
    \hfill
    \begin{minipage}{0.24\textwidth}
        \centering
        \includegraphics[width=\linewidth]{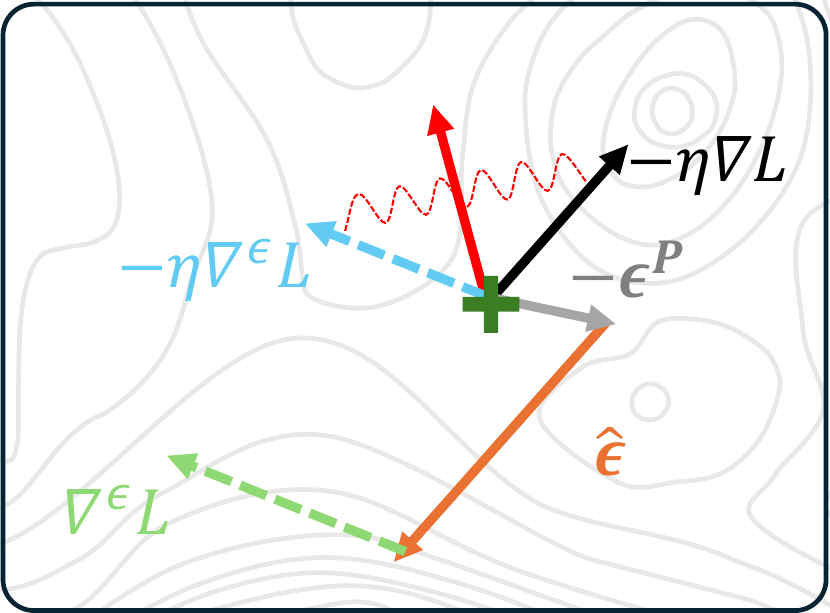}
        \caption*{\fontsize{7pt}{7pt}\selectfont{{(b) SAMPLe's early stages of optimization.}}}
    \end{minipage}
    \hfill
    \begin{minipage}{0.24\textwidth}
        \centering
        \includegraphics[width=\linewidth]{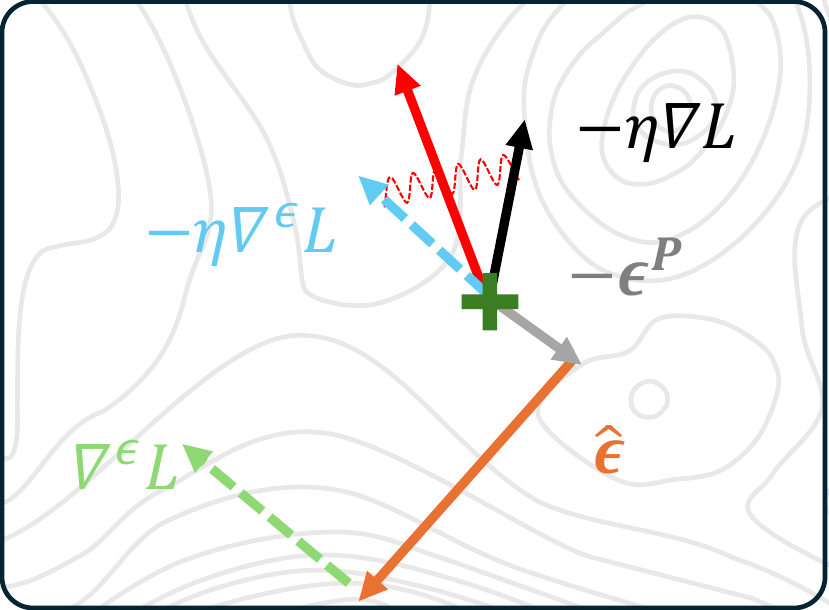}
        \caption*{\fontsize{6.75pt}{6.75pt}\selectfont{{(c) SAMPLe's middle stage of optimization.}}}
    \end{minipage}
    \hfill
    \begin{minipage}{0.24\textwidth}
        \centering
        \includegraphics[width=\linewidth]{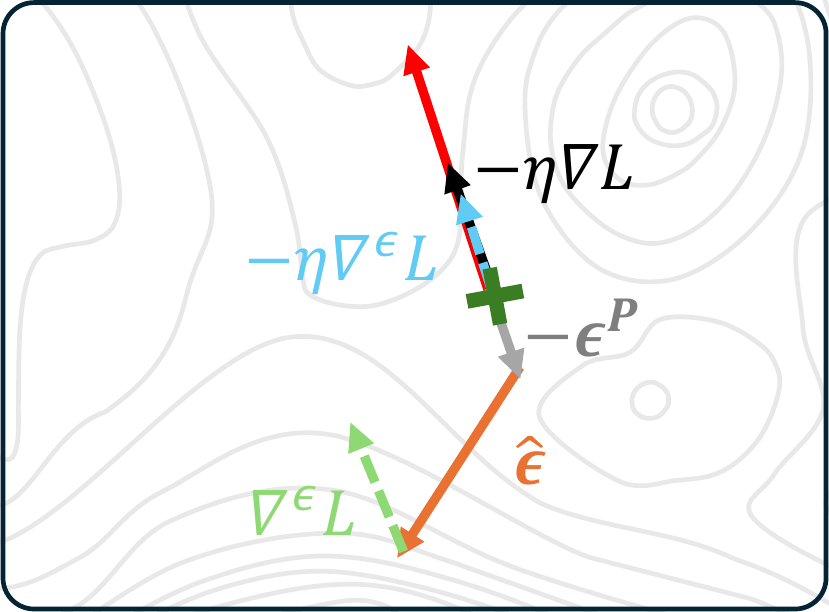}
        \caption*{\fontsize{7pt}{7pt}\selectfont{{(d) SAMPLe's last stages of optimization.}}}
    \end{minipage}
    \begin{minipage}{\textwidth}
        \centering
        \includegraphics[width=0.99\linewidth]{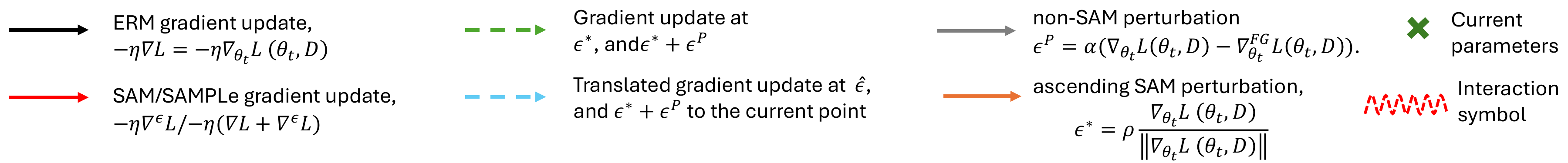}
    \end{minipage}
    \vspace{-0.3cm}
    \caption{\textbf{SAMPLe vs. vanilla SAM:} Unlike SAM, which applies uniform gradient updates, SAMPLe dynamically adjusts gradients across stages, promoting smoother optimization in early stages and robust convergence in later stages, resulting in improved generalization and performance.}
    \label{fig:sample_comparison}
\end{figure*}

\section{Related Work }
\label{sec:rel_wrk}
\vspace{-0.3cm}

\textbf{Prompt learning} is a powerful approach for adapting large pre-trained VLMs like CLIP \cite{radford2021learning}, ALIGN \cite{jia2021scaling}, LiT \cite{zhai2022lit}, and FILIP \cite{yao2021filip} to downstream tasks. These models, trained on large image-text datasets using contrastive learning \cite{chen2020simple}, excel at general representation learning \cite{zhang2024vision}, but adapting them to specific tasks, especially in low-data settings like few-shot learning \cite{zhou2022conditional}, remains challenging. Traditional fine-tuning methods are computationally expensive and prone to overfitting \cite{khattak2023self}. Prompt learning addresses these issues by introducing learnable prompts, enabling efficient task adaptation without retraining the entire model. CoOp \cite{zhou2022learning} fine-tunes continuous text prompts for few-shot recognition but struggles with unseen classes. CoCoOp \cite{zhou2022conditional} improves zero-shot generalization by conditioning prompts on image features, while MaPLe \cite{khattak2023self} optimizes prompts across both vision and language branches. Methods like PromptSRC \cite{khattak2023self} and ProDA \cite{lu2022prompt} further enhance generalization using regularization and diverse task-related representations. Recent works extend this line with textual-based class-aware prompt tuning, TCP \cite{yao2024tcp}), decoupled embedding parameterization, DEPT \cite{iacob2025dept}), overfitting-aware prompt regularization, LOBG \cite{ding2024lobg}), and diversity covariance-aware prompt learning \cite{zhou2025diversity}.

\noindent\textbf{SAM} is an effective optimization framework designed to improve model generalization by finding flatter minima in the loss landscape \cite{foret2021sharpnessaware}. SAM has been successfully applied in language modeling \cite{bahri2021sharpness}, fluid dynamics \cite{jetly2022splash}, medical imaging \cite{yeo2008effects}, multimodal learning \cite{rajoli2026modality}, and Reinforcement Learning (RL) \cite{lotfi2025task}. Building on the original framework, several extensions have refined SAM’s adaptability. F-SAM \cite{li2024friendly} utelizes a less aggressive perturbation that leads to better generalization, while SAGM \cite{wang2023sharpness} enhances gradient alignment to balance loss minimization and flatness. GCSCoOP\cite{10814656}, a recent SAM-based approach, aims to address generalization by considering both loss sharpness and value but remains dependent on manually set hyperparameters. Adaptive SAM (ASAM) adjusts the sharpness radius dynamically \cite{kwon2021asam}, GSAM simplifies sharpness measurement using surrogate gap calculations \cite{zhuang2022surrogate}, and Fisher SAM leverages Fisher information for sharper neighborhood estimations \cite{kim2022fisher}. GAM further improves generalization by focusing on the maximal gradient norm \cite{zhang2023gradient}.
In vision tasks, SAM has demonstrated notable success in models like Vision Transformers and MLP-Mixers \cite{chen2021vision}, enhancing both accuracy and robustness. However, SAM’s potential remains insufficiently explored in multi-modal learning. \vspace{-0.5cm}
\section{Preliminaries}\vspace{-0.3cm}
\label{sec:prlm}

\noindent\textbf{Prompt Learning in VLM:}  
Prompt learning aims to enhance the adaptability of VLMs by learning task-specific prompts. This approach involves introducing tunable token vectors $\mathbf{v} = \{ v_i \}_{i=1}^{N}$, where the text inputs for each class are represented as $t_m = {v_1, \dots, v_N, c_m}$, with $c_m$ denoting the class name and where $N+1$ defines number of tokens for text input. These token vectors are optimized using cross-entropy loss $\mathcal{L}_{CE}$, while keeping the parameters of the underlying VLM model, such as CLIP, frozen: 
\begin{equation} 
    {L}_{CE}(\mathbf{v}) = - \sum_{m} y_m \log p(m|x), \label{eq} 
\end{equation} 
\begin{equation} 
    p(m|x) = \frac{\exp(\text{cos}(\mathcal{I}(x), \mathcal{T}(t_m)) / \tau)}{\sum_{j=1}^M \exp(\text{cos}(\mathcal{I}(x), \mathcal{T}(t_j)) / \tau)}.
    \label{eq} 
\end{equation} 
The learned token vectors enable the model to adapt to specific tasks by refining how class names are represented as input to the model, improving the prompt's alignment with task requirements.

\noindent\textbf{Sharpness-aware Prompt Learning:}
SAM is a recently proposed optimization framework designed to enhance generalization by avoiding sharp minima in the loss landscape~\cite{foret2020sharpness}.

Consider a deep neural network $f(x; \theta)$ parameterized by $\theta \in \mathbb{R}^d$, with a loss function $\ell(f(x; \theta), y)$ that measures the discrepancy between the model’s prediction $f(x; \theta)$ and the true label $y$. For a dataset $\mathcal{D} = {(x_i, y_i)}_{i=1}^N$, the empirical loss over the dataset is given by:
\begin{equation}
    {L}(\theta; \mathcal{D}) = \frac{1}{N} \sum_{i=1}^N \ell(f(x_i; \theta), y_i).
    \label{eq:empirical_loss}
\end{equation}

Optimization methods like stochastic gradient descent (SGD) often converge to sharp minima that generalize poorly. SAM mitigates this by solving the following min-max optimization problem:
\begin{equation}
    \min_{\theta} \max_{\|\epsilon\|_2 \leq \rho} {L}(\theta + \epsilon; \mathcal{D}),
    \label{eq:minmax_sam}
\end{equation}
where $\epsilon \in \mathbb{R}^d$ is a perturbation constrained by $\|\epsilon\|_2 \leq \rho$, with $\rho$ controlling the radius of the perturbation around the parameter $\theta$. 
Because solving the inner maximization directly is computationally expensive, SAM uses a first-order Taylor expansion to approximate the perturbation $\epsilon^\star$:
\begin{equation}
    \epsilon^\star \approx \rho \cdot \frac{\nabla {L}(\theta; \mathcal{D})}{\|\nabla {L}(\theta; \mathcal{D})\|_2}.
    \label{eq:sam_perturbation}
\end{equation}
The network parameters are updated as follows:
\begin{equation}
    \theta_{t+1} \leftarrow \theta_{t} - \eta \cdot \nabla_{} \mathcal{L}(\theta_{t} + \epsilon^\star; \mathcal{D})
    \label{eq:sam_update_rule},
\end{equation}

\noindent where $\eta$ is the learning rate, and $\epsilon^\star$ is the perturbation that guides the model towards flatter minima in Fig.~\ref{fig:sample_comparison}a.

The limited number of learnable parameters in prompt learning often makes traditional approaches susceptible to overfitting or getting trapped in sharp loss surfaces \cite{zhu2023prompt}. SAM addresses this challenge by guiding the optimization process toward flatter minima in the loss landscape, enhancing the model’s ability to generalize to unseen distributions and domains. By integrating SAM into the VLM prompt learning frameworks, we effectively improve the model's generalization across unseen domains while maintaining good performance on the training dataset.

\subsection{Full-Batch Bias Hinders SAM}
\label{sec:full-batch}
One of the core ideas behind SAM is its perturbation vector, $\epsilon$, which plays a crucial role in guiding the model toward flatter minima, thereby enhancing generalization. When applied effectively, the perturbation vector steers the optimization process toward regions of the loss landscape that are less sensitive to sharp parameter changes. However, when the perturbation vector is computed as if all training samples were processed in a single batch (i.e., a full-batch perturbation), it can diminish the natural stochasticity of a mini batch based approach. This stochasticity is crucial for balancing exploration across the loss landscape. As noted by \cite{li2024friendly}, this lack of stochastic variability in full-batch perturbation can limit generalization, even resulting in worse performance compared to SGD in some cases.

While the full gradient provides a global view of the entire dataset, this global smoothness often comes at the expense of exploration, particularly when the perturbation $\epsilon$ is excessively aligned with the full gradient. This alignment suppresses the stochastic variability introduced by mini-batch gradients, leading to over-smoothing during training. As a result, the model may converge to sharper minima, reducing generalization. This phenomenon can cause SAM to underperform even compared to traditional optimization methods like SGD \cite{li2024friendly}. The full-gradient perturbation in the SAM context captures the overall trend of the loss function across the entire dataset. The projection of the mini-batch gradient onto the full batch gradient demonstrated in Fig.~\ref{fig:projection} (in Appendix) can be formulated as: 

\begin{equation}
    \begin{aligned}
\text{Proj}_{\nabla^{\mathcal{F}}_{} {L}(\theta; \mathcal{D})}  \nabla_{} {L}_{}(\theta; \mathcal{D}) = 
    \frac{\|\nabla {L}_{}(\theta; \mathcal{D})\|_2}{\|\nabla^{\mathcal{F}}_{} {L}_{}(\theta; \mathcal{D})\|_2}
    \cos\big(\nabla^{\mathcal{F}}_{} {L}(\theta; \mathcal{D}), \nabla_{} {L}_{}(\theta; \mathcal{D})\big) \nabla_{} {L}^{\mathcal{F}}_{\mathcal{}}(\theta; \mathcal{D}),
    \end{aligned}
    \label{eq:fg_proj}
\end{equation}
where $\cos(\cdot)$ represents the cosine similarity between the mini-batch gradient, $\nabla_{} {L}_{}(\theta; \mathcal{D})$, and the full batch gradient vector, $\nabla^{\mathcal{F}}_{} {L}(\theta; \mathcal{D})$. The batch-specific gradient is then defined as:

\begin{equation}
    \nabla^{\mathcal{B}}_{} {L}(\theta; \mathcal{D}) = \nabla_{} {L}_{}(\theta; \mathcal{D}) - \sigma ~ \xi ~ \nabla^{\mathcal{F}}_{} {L}(\theta; \mathcal{D}),
    \label{eq:B_proj}
\end{equation}

\noindent where $\xi$ and $\sigma = \frac{\|\nabla{L}_{}(\theta; \mathcal{D})\|_2}{\|\nabla^{\mathcal{F}}_{} {L}_{}(\theta; \mathcal{D})\|_2}$ are cosine similarity operator and normalization factor, respectively. This batch-specific gradient represents the component of the mini-batch gradient orthogonal to the full batch gradient. 
Due to the computational overhead of computing the full gradient over the entire dataset, it is typically approximated using an exponentially moving average (EMA) of the mini-batch gradients:

\begin{equation}
    m_t = \lambda m_{t-1} + (1 - \lambda) \nabla_{} {L}(\theta; \mathcal{D}),
    \label{eq:EMA}
\end{equation}

where $m_t$ approximates $\nabla^{\mathcal{F}}_{} {L}(\theta; \mathcal{D})$ if $t$ is large enough, and $\lambda$ controls the influence of past gradients \cite{li2024friendly}.
The full-gradient component biases the perturbation vector $\epsilon$ based on the entire dataset’s loss landscape, leading to smoother updates but reducing the exploration necessary for finding flatter minima. This smoothing effect can hinder SAM’s generalization performance by diminishing the stochastic variability introduced by mini-batch gradients \cite{li2024friendly}. 

\noindent Therefore, balancing the stability offered by full-batch gradient updates with the exploration provided by batch-specific gradients is essential for maintaining SAM's effectiveness. \vspace{-0.3cm}
\section{SAMPLe}
\label{sec:SAMPLe} \vspace{-0.2cm}
Given the challenges of VLM prompt learning, we introduce two pivotal challenges to ensure that the model performs robustly across both seen and unseen distributions: (i) \textbf{\textit{Optimal minima}:} The learned prompts must not only identify a flat minimum in the loss surface but also ensure that this minimum corresponds to a sufficiently low loss value, guaranteeing well-optimized performance on the training samples. (ii) \textbf{\textit{Generalization to unseen domains}:} While prompts may perform well on the training data, it is critical to ensure they generalize effectively to unseen distributions. This challenge is amplified by the significantly smaller number of parameters in prompts compared to overparameterized neural networks, making it more challenging to train models that generalize across diverse domains. To address this, the optimization strategy enforced by the objective function must inherently balance exploitation and exploration, dynamically adapting to the current state of the optimization. This ensures that the learning process acquires network parameters capable of achieving both a sufficiently low loss on the training data and robust generalization to unseen domains and distributions. Inspired by \cite{wang2023sharpness} and designed to satisfy these two conditions simultaneously, we propose the SAMPLe objective function to optimize the training of learnable prompts as follows:
\begin{equation}
    \begin{aligned}
        \min_{\theta} \mathcal{L}(\theta; \mathcal{D})~ =
        &\min_{\theta} \Big[{L}_{}(\theta; \mathcal{D}) + {L}_{}\left(\theta + {\epsilon}^{\star} - \alpha (\nabla^{\mathcal{B}}_{} \mathcal{L}(\theta; \mathcal{D}))\right)\Big] =\\
        &\min_{\theta} ~ \Big[{L}_{}(\theta; \mathcal{D}) +
         {L}_{}\left(\theta + {\epsilon}^{\star} - \alpha (\nabla {L}_{}(\theta; \mathcal{D}) - \xi ~ \sigma ~ \nabla^{\mathcal{F}} {L}(\theta; \mathcal{D})); \mathcal{D}\right)\Big],
    \end{aligned}
    \label{eq:obj_fn_1}
\end{equation}
where the optimal perturbation is defined by Eq.~\ref{eq:sam_perturbation} as ${\epsilon}^{\star} = \rho \cdot \frac{\nabla {L}_{}(\theta; \mathcal{D})}{\|\nabla {L}_{}(\theta; \mathcal{D})\|_2}$.
\vspace{-0.2cm}
\subsection{SAMPLe Analysis and Algorithm:}
\label{sec:SAMPLe_analysis}\vspace{-0.2cm}
This section presents an analysis to clarify the proposed SAMPLE method. Applying the Taylor expansion reveals that using $\nabla^{\mathcal{B}} L(\theta; \mathcal{D})$ in the second term of the loss function implies two critical conditions that enhance the learning process: \vspace{-0.2cm}
\begin{equation}
    \begin{aligned}
        &\min_{\theta} {L}_{p}\left(\theta - \alpha \nabla^{\mathcal{B}}_{} {L}(\theta; \mathcal{D})\right)
        = \\
        &\min_{\theta} \Big[{L}_{p}(\theta; \mathcal{D}) - \alpha \big(\nabla_{} {L}_{p}(\theta; \mathcal{D}) \cdot \nabla_{} {L}_{}(\theta; \mathcal{D})\big)~+ \\
        & \alpha \xi \sigma \big(\nabla_{} {L}_{p}(\theta; \mathcal{D}) \cdot \nabla_{}^{\mathcal{F}} {L}(\theta; \mathcal{D})\big) + R(\nabla_{}^{\mathcal{B}} {L}(\theta; \mathcal{D})^2) \Big]\approx \\
        &\min_{\theta} \Big[{L}_{p}(\theta; \mathcal{D}) - \alpha \nabla_{} {L}_{p}(\theta; \mathcal{D}) \cdot \nabla_{} {L}_{}(\theta; \mathcal{D}) +
        \alpha \xi \sigma \nabla_{} {L}_{p}(\theta; \mathcal{D}) \cdot \nabla_{}^{\mathcal{F}} {L}(\theta; \mathcal{D})\Big]
    \end{aligned}
    \label{eq:taylor_expansion_1}
\end{equation}

\vspace{-0.3cm}
where $L_P$, represents the loss at the perturbed point in parameter space;

\begin{equation}
    \begin{aligned}
        \min_{\theta} \Big[&{L}(\theta; \mathcal{D}) + {L}_{p}(\theta; \mathcal{D}) - 
         \alpha \big(\nabla_{} {L}_{p}(\theta; \mathcal{D}) \cdot \nabla_{} {L}_{}(\theta; \mathcal{D})\big)~+\\
        &  \alpha \xi \sigma \big(\nabla_{} {L}_{p}(\theta; \mathcal{D}) \cdot \nabla_{}^{\mathcal{F}} {L}(\theta; \mathcal{D})\big)\Big].\vspace{-0.2cm}
    \end{aligned}
    \label{eq:aligned_sam_objective}
\end{equation}

\noindent The objective function can be interpreted as follows:\\
\textbf{First}, minimizing the loss at the current point ($L$) and the perturbed neighboring point ($L_p$).  
\textbf{Second}, considering the alignment between the gradient at the perturbed point and the gradient at the current point, represented as $-\alpha \big(\nabla L_{p}(\theta; \mathcal{D}) \cdot \nabla L(\theta; \mathcal{D})\big)$, which we refer to as \textit{exploitation}.  
\textbf{Third}, ensuring that the gradient at the perturbed point is orthogonal to the full-batch gradient or aligned with the batch-specific gradient (see Fig.~\ref{fig:projection} in Appendix and subsection~\ref{sub:orthog_anlys}), represented as $+\big(\nabla L_{p}(\theta; \mathcal{D}) \cdot \nabla^{\mathcal{F}} L(\theta; \mathcal{D})\big)$, which we refer to as \textit{exploration}.

\begin{algorithm}[t]
\small
\caption{SAMPLe algorithm}
\label{alg:sagm}
\begin{algorithmic}[1]
\Require Few-shot training set: $\mathcal{D} = \bigcup_{i=1}^N \{ \bigcup_{m=1}^M {(x^{i}_{m}, y^{i})}\}\} $, number of classes $M$, class name $\mathbf{c} = \{c_j\}_{j=1}^{M}$, pre-trained CLIP with image-encoder $\mathcal{I}$ and text encoder $\mathcal{T}$, prompt token length $P$, training epoch $E$, learning rate $\eta$, SAM perturbation radius $\rho$, $\alpha$, weight decay coefficient$\lambda$, batch-size, $b$.

\Ensure trained learnable prompts.
\State Initialize parameters $\theta_0$, $t = 0$
\While{not converged}
    \State Sample mini-batch $\mathcal{B} \in \mathcal{D}$;
    \State Compute the training mini-batch loss gradient $\nabla {L}(\theta_t; \mathcal{D})$;
    \State update the full-gradient, $m_{t} \approx \nabla^{\mathcal{F}}L(\theta_t; \mathcal{D})$ using Eq.~\ref{eq:EMA} 
    \State Compute SAM perturbation $\epsilon^{*}_{t}$,  according to Eq.~\ref{eq:sam_perturbation}:
    \State Compute $\xi_t = \cos \big(\nabla^{\mathcal{F}}_{} {L}(\theta_t; \mathcal{D}), \nabla_{} {L}_{}(\theta_t; \mathcal{D})\big) $.
    \State Compute normalization factor, $\sigma_t = \frac{\|\nabla \mathcal{L}_{}(\theta; \mathcal{D})\|_2}{\|\nabla^{\mathcal{F}}_{} {L}_{}(\theta; \mathcal{D})\|_2}$.
    \State Compute objective optimization, $\mathcal{L}(\theta_t; \mathcal{D})$, using Eq.~\ref{eq:obj_fn_1}
    \State Update weights:
    $\theta_t \gets \theta_t - \eta_{t} \nabla \mathcal{L} (\theta_{t}; \mathcal{D})$

    \State $t \gets t + 1$
\EndWhile
\end{algorithmic}
\end{algorithm}

\vspace{-0.3cm}
\vspace{-0.2cm}
\subsection{Gradient Orthogonality Analysis:}
\label{sub:orthog_anlys}
\vspace{-0.2cm}
\noindent The third term of Eq.~\ref{eq:aligned_sam_objective} is designed to drive $\nabla {L}_{p}(\theta; \mathcal{D})$ orthogonal to $\nabla^{\mathcal{F}}_{} {L}(\theta; \mathcal{D})$ at the original point $\theta_t$, as we establish below.
\noindent By Eq.~\ref{eq:B_proj}, the mini-batch gradient decomposes as 
$\nabla {L}_{}(\theta; \mathcal{D}) = \nabla^{\mathcal{B}}_{} 
{L}(\theta; \mathcal{D}) + \xi \sigma \nabla^{\mathcal{F}}_{} 
{L}(\theta; \mathcal{D})$, so the second term of 
Eq.~\ref{eq:aligned_sam_objective} alone would pull 
$\nabla {L}_{p}(\theta; \mathcal{D})$ toward both components. 
The third term introduces anti-alignment with $\nabla^{\mathcal{F}}_{} 
{L}(\theta; \mathcal{D})$ (Eq.~\ref{eq:obj_fn_1}): at every iteration, 
the second term drives $\nabla {L}_{p}(\theta; \mathcal{D})$ toward 
$\nabla {L}_{}(\theta; \mathcal{D})$ for performance on seen 
distributions, while the third term resists the full-batch component 
$\xi \sigma \nabla^{\mathcal{F}}_{} {L}(\theta; \mathcal{D})$ within 
that alignment. The equilibrium of these two forces steers 
$\nabla {L}_{p}(\theta; \mathcal{D})$ exclusively toward 
$\nabla^{\mathcal{B}}_{} {L}(\theta; \mathcal{D})$. Since 
$\nabla^{\mathcal{B}}_{} {L}(\theta; \mathcal{D}) \perp 
\nabla^{\mathcal{F}}_{} {L}(\theta; \mathcal{D})$ by construction 
(Eq.~\ref{eq:B_proj} and Fig.~\ref{fig:projection}, Appendix), 
this equilibrium guarantees orthogonality of 
$\nabla {L}_{p}(\theta; \mathcal{D})$ to 
$\nabla^{\mathcal{F}}_{} {L}(\theta; \mathcal{D})$ at $\theta_t$.

\noindent Substituting $\nabla^{\mathcal{B}}_{} {L}(\theta; \mathcal{D}) 
= \nabla {L}_{}(\theta; \mathcal{D}) - \xi \sigma \nabla^{\mathcal{F}}_{} 
{L}(\theta; \mathcal{D})$ (Eq.~\ref{eq:B_proj}) into the joint second 
and third terms of Eq.~\ref{eq:aligned_sam_objective}:

\begin{small}
\begin{align}
    &-\alpha\big(\nabla {L}_{p}(\theta; \mathcal{D}) \cdot 
    \nabla {L}_{}(\theta; \mathcal{D})\big)
    +\alpha\xi\sigma\big(\nabla {L}_{p}(\theta; \mathcal{D}) \cdot 
    \nabla^{\mathcal{F}}_{} {L}(\theta; \mathcal{D})\big)\nonumber\\
    &= -\alpha\big(\nabla {L}_{p}(\theta; \mathcal{D}) \cdot 
    \nabla^{\mathcal{B}}_{} {L}(\theta; \mathcal{D})\big) \\ 
    &~~~-\underbrace{\alpha\xi\sigma\big(\nabla {L}_{p}(\theta; \mathcal{D}) 
    \cdot \nabla^{\mathcal{F}}_{} {L}(\theta; \mathcal{D})\big) 
    + \alpha\xi\sigma\big(\nabla {L}_{p}(\theta; \mathcal{D}) 
    \cdot \nabla^{\mathcal{F}}_{} {L}(\theta; 
    \mathcal{D})\big)}_{=\,0}\nonumber\\
    &= -\alpha\big(\nabla {L}_{p}(\theta; \mathcal{D}) \cdot 
    \nabla^{\mathcal{B}}_{} {L}(\theta; \mathcal{D})\big).
    \label{eq:orthog_cancel}
\end{align}
\end{small}

\noindent The $\nabla^{\mathcal{F}}_{} {L}(\theta; \mathcal{D})$ 
component cancels, and the objective reduces to maximizing 
$\nabla {L}_{p}(\theta; \mathcal{D}) \cdot \nabla^{\mathcal{B}}_{} 
{L}(\theta; \mathcal{D})$ exclusively. Since $\nabla^{\mathcal{B}}_{} 
{L}(\theta; \mathcal{D}) \perp \nabla^{\mathcal{F}}_{} {L}(\theta; 
\mathcal{D})$ by construction (Eq.~\ref{eq:B_proj}), 
$\nabla {L}_{p}(\theta; \mathcal{D})$ is driven toward the subspace 
orthogonal to $\nabla^{\mathcal{F}}_{} {L}(\theta; \mathcal{D})$ 
at $\theta_t$. To confirm this is preserved under $\epsilon^{\star}$, 
let $\hat{\theta}_t = \theta_t + \epsilon^{\star}_t - \alpha_t 
\nabla^{\mathcal{B}}_{} {L}(\theta_t; \mathcal{D})$ denote the 
perturbed point. By the triangle inequality, $\|\epsilon^{\star}_t\| 
\leq \rho_t$ (Eq.~\ref{eq:sam_perturbation}) and condition~(i) in 
Sec.~\ref{sub:conv_anlys}:
\begin{equation}
    \|\hat{\theta}_t - \theta_t\| \leq \rho_t 
    + \alpha_t \nabla\mathcal{L}_{\max}.
    \label{eq:orthog_bound}
\end{equation}

\noindent By K-Lipschitz gradient condition~(ii) in 
Sec.~\ref{sub:conv_anlys} and Lemma~\ref{lemma1} 
(Appendix, Proof of Theorem~\ref{theorem1}):
\begin{small}
\begin{align}
    \bigl|(\nabla {L}_{p}(\theta; \mathcal{D}) {-} 
    \nabla {L}_{}(\theta; \mathcal{D}))
    {\cdot}\nabla^{\mathcal{F}}_{} {L}(\theta; \mathcal{D})\bigr|
    &\leq K(\rho_t {+} \alpha_t\nabla\mathcal{L}_{\max})
    \cdot\nabla {L}_{\max}(1{-}\lambda^T).
    \label{eq:orthog_vanish}
\end{align}
\end{small}
\noindent As $\rho_t, \alpha_t = \mathcal{O}(1/\sqrt{t})$ and 
$\lambda^T \to 0$ as $t \to \infty$, the right-hand side vanishes, 
confirming orthogonality is preserved throughout training.
\vspace{-0.2cm}

\subsection{Convergence of SAMPLe:}
\label{sub:conv_anlys}
Assuming that the objective function defined in Eq.~\ref{eq:obj_fn_1} satisfies the following conditions; 

(i) The gradient of the loss function $\nabla \mathcal{L}(\theta_t; \mathcal{D})$ is bounded, i.e., \vspace{-0.2cm}
\begin{equation}
    \|\nabla \mathcal{L}(\theta_t; \mathcal{D})\| \leq \nabla\mathcal{L}_{max}, \quad \forall t.
\end{equation}

(ii) The stochastic gradient is K-Lipschitz, i.e., 
\begin{equation}
    \|\nabla \mathcal{L}(\theta_t; \mathcal{D}) - \nabla \mathcal{L}(\theta_t'; \mathcal{D})\| \leq K \|\theta_t - \theta_t'\|, \forall (\theta_t, \theta_t').\vspace{-0.2cm} 
\end{equation}
\\
Let the learning rate and both perturbations radii be defined as $\eta_t = \frac{\eta_0}{\sqrt{t}}$, $\alpha_t = \frac{\alpha_0}{\sqrt{t}}$, and $\rho_t = \frac{\rho_0}{\sqrt{t}}$, respectively. 
Theorem. ~\ref{theorem1} in Appendix.~\ref{sec: app_theorm 1} proves that the objective function satisfies following inequality\vspace{-0.2cm}
\begin{equation}
    \frac{1}{T} \sum_{t=1}^{T}  \left[ \|\nabla \mathcal{L}(\theta_t; \mathcal{D})\|^2 \right] \leq \mathcal{O}\left(\frac{\log T}{\sqrt{T}}\right),
    \label{eq:convergense1}
\end{equation}
It means the objective function defined in Eq.~\ref{eq:obj_fn_1} converges with a rate of $\mathcal{O}\left(\frac{\log T}{\sqrt{T}}\right)$, which is comparable with optimization methods such as SGD and Adam.
\vspace{-0.4cm}
\section{Experiments}
\label{sec:Exp}
\vspace{-0.2cm}

In this section, we elaborate on the datasets, baselines, and implementation details utilized in our study. Next, we evaluate the proposed method on base-to-new class generalization (Sec. \ref{sec:base_to_novel}), cross-dataset generalization (Sec.~\ref{sec:cross_dataset}), and cross-domain generalization (Sec.~\ref{sec:cross_domain}).


\begin{figure*}[t]
    \centering
    \begin{minipage}{0.75\textwidth}
        \centering
        \fbox{ 
            \begin{minipage}{0.95\textwidth}
                \centering
                \begin{minipage}{0.32\textwidth}
                    \centering
                    \includegraphics[width=\linewidth]{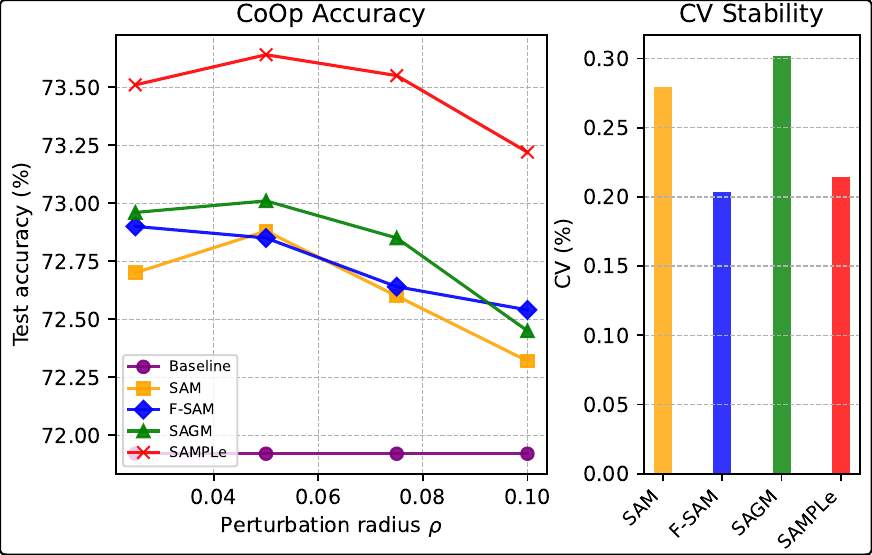}
                \end{minipage}
                \hfill
                \begin{minipage}{0.32\textwidth}
                    \centering
                    \includegraphics[width=\linewidth]{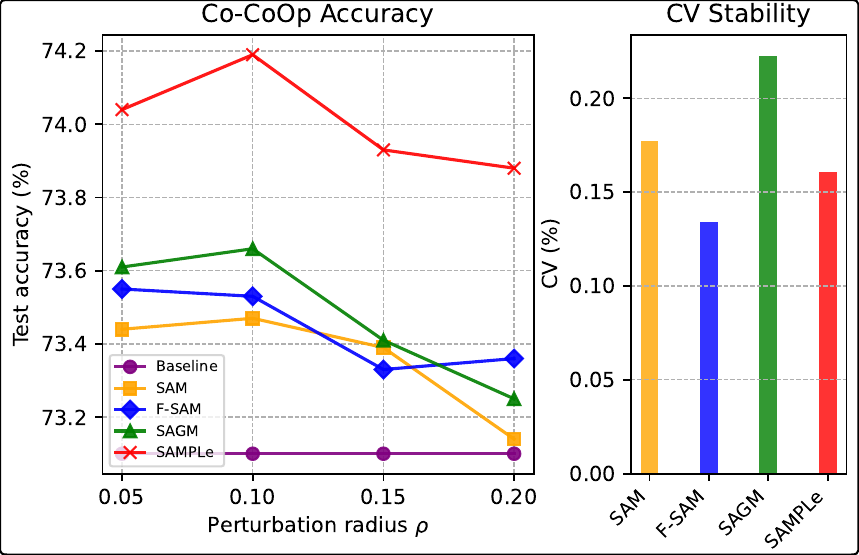}
                \end{minipage}
                \hfill
                \begin{minipage}{0.32\textwidth}
                    \centering
                    \includegraphics[width=\linewidth]{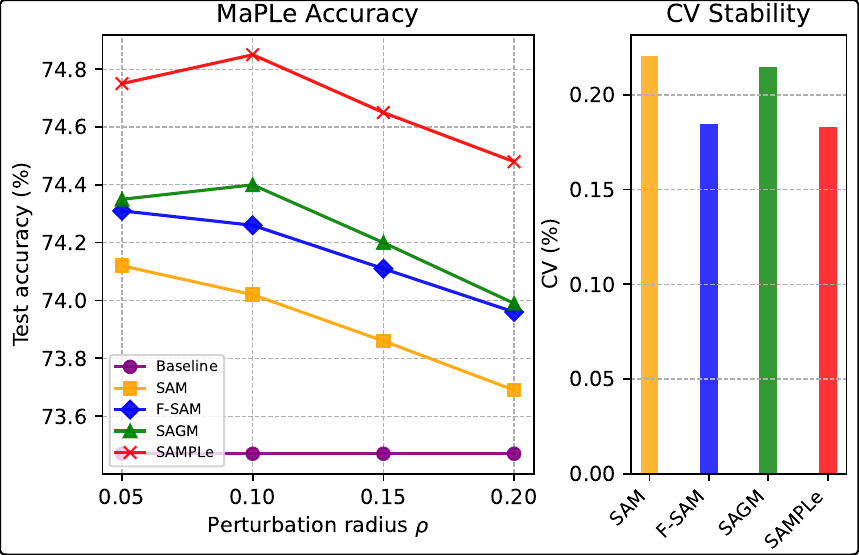}
                \end{minipage}
            \end{minipage}
        }
        \vspace{0.2cm}
        \subcaption{\fontsize{7pt}{7pt}\selectfont{Comparison of F-SAM, SAM, SAGM, and SAMPLe across $\rho$ values, deployed on CoOp, CoCoOp, and MaPLe.}}
    \end{minipage}
    \hfill
    \begin{minipage}{0.24\textwidth}
        \centering
        \fbox{ 
            \begin{minipage}{0.95\textwidth}
                \centering
                \includegraphics[width=\linewidth]{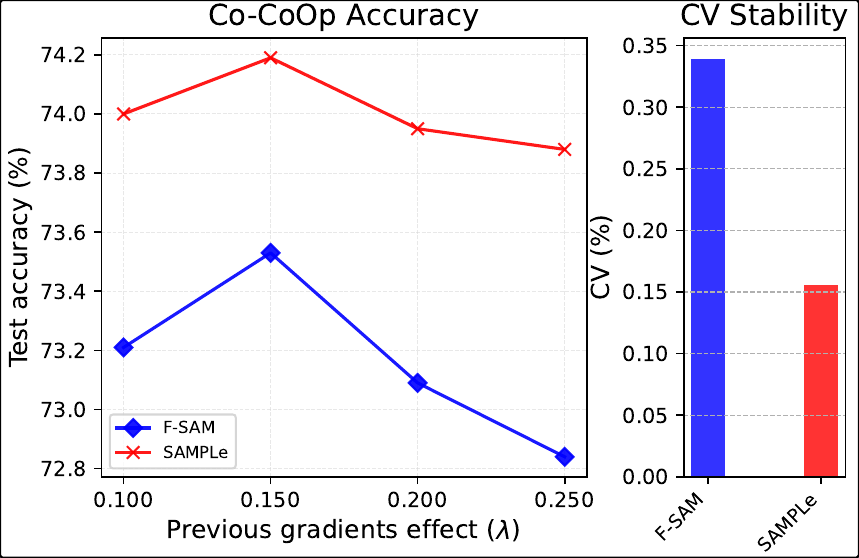}
            \end{minipage}
        }
        \vspace{0.2cm}
        \subcaption{\fontsize{7pt}{7pt}\selectfont{SAMPLe vs F-SAM across $\lambda$ values.}}
    \end{minipage}
    \vspace{0.3cm}
    \caption{Accuracy and coefficient of variation of SAM, F-SAM, SAGM, and SAMPLe on ImageNet across different values of $\rho$ and $\lambda$ for various prompt learning methods, including CoOp, Co-CoOp, and MaPLe.}
    \label{fig:comparison}
\end{figure*}
\vspace{-0.3cm}
\subsection{Datasets}
\label{ssec:datast}
In this work, we utilize $11$ publicly available image classification datasets as downstream tasks: ImageNet \cite{deng2009imagenet}, Caltech101 \cite{1384978}, OxfordPets \cite{parkhi2012cats}, StanfordCars \cite{krause20133d}, Flowers102 \cite{nilsback2008automated}, Food101 \cite{bossard2014food}, FGVCAircraft \cite{maji2013fine}, SUN397 \cite{xiao2010sun}, DTD \cite{cimpoi2014describing}, EuroSAT \cite{helber2019eurosat}, and UCF101 \cite{soomro2012ucf101}. Also, four datasets of ImageNetV2 \cite{recht2019imagenet}, ImageNet-Sketch \cite{wang2019learning}, ImageNet-A \cite{hendrycks2021natural}, and ImageNet-R \cite{hendrycks2021many} serve exclusively as target domains for cross-domain generalization. Detailed dataset descriptions are provided in Supplementary~
\ref{ssec:datatsets_Appndx}.
For each baseline we strictly adhere to the configurations reported in the original papers, including learning rate schedules, backbone architecture, prompt length, prompt initialization, few-shot setup, and three random seeds. Detailed of hyperparameter and implementation descriptions are provided in Supplementary~
\ref{deployment}.
\vspace{-0.3cm}

\subsection{Base-to-New Generalization Setting}
\label{sec:base_to_novel}
To evaluate the generalization capability of our method, we train the model on base classes and test its performance on both base and novel classes. This experimental setup is designed to evaluate the trade-off between maintaining strong performance on the base classes while adapting effectively to novel classes. The results are presented in Table~\ref{tab:full_camparison}, with the harmonic mean (HM) of base and novel accuracies serving as the primary metric for comparison.
As shown in Table~\ref{tab:full_camparison}, integrating SAMPLe consistently improves the harmonic mean (HM) across different prompt-learning backbones and datasets, indicating a more balanced generalization between base and novel classes. On the averaged results (Table~\ref{tab:full_camparison}(a)), SAMPLe achieves the highest HM for all methods. For example, CoOp improves from 77.65 (+SAGM) to 78.88 with SAMPLe, CoCoOp from 78.28 to 79.51, MAPLe from 79.96 to 80.28, CoPrompt from 81.49 to 82.02, and TCP from 81.40 to 81.95. These improvements are mainly associated with higher accuracy on the new classes while maintaining competitive performance on the base classes.
A similar trend is observed across individual datasets. On Flowers102, SAMPLe provides the best HM across all backbones, for instance improving CoCoOp from 85.26 (+SAGM) to 87.15 and TCP from 86.48 to 87.02. On DTD, SAMPLe increases the HM from 65.46 to 69.77 for CoOp and from 70.93 to 71.37 for MAPLe. On EuroSAT, SAMPLe again achieves the highest HM values, reaching 82.44 for CoOp and 84.19 for TCP. Overall, the consistent improvements across methods and datasets suggest that SAMPLe leads to more stable optimization behavior and improved generalization to novel classes without drastically degrading base-class performance.
Another highlighted fact is that, superiority of SAMPLe is almost clear over all mentioned methods specifically in HM, which means SAMPLe keeps both Base and New performance and does not sacrify one of them in favor of the other.

\textbf{Comparison with State-of-the-Art Methods:}
In addition to comparisons with SAM-based optimizers, we evaluate CoOp+SAMPLe against several representative prompt-learning approaches, including CoOp, CoCoOp, MAPLe, and CoPrompt. As shown in Table~\ref{tab:full_camparison}, CoOp+SAMPLe achieves competitive and often superior harmonic mean (HM) performance across several datasets when compared with the corresponding naive prompt-learning methods. 
For example, on OxfordPets, CoOp+SAMPLe achieves an HM of 97.08, surpassing the naive variants of CoCoOp (96.43), MAPLe (96.58), and CoPrompt (96.87). This result suggests that SAMPLe can effectively improve generalization while maintaining strong base-class performance. On more challenging datasets such as FGVCAircraft, CoOp+SAMPLe obtains an HM of 37.58, outperforming MAPLe (36.50) and approaching the performance of CoPrompt (39.76). These results indicate that the proposed optimizer remains effective even in fine-grained recognition scenarios where intra-class variations are large and training samples are limited. 
Similarly, on StanfordCars, CoCoOp+SAMPLe achieves an HM of 75.51, outperforming CoCoOp+FSAM (74.85) and CoCoOp+SAGM (74.75). While maintaining comparable base accuracy, SAMPLe improves performance on the novel classes (77.89), leading to a higher harmonic mean. 
More broadly, the averaged results in Table~\ref{tab:full_camparison}(a) reveal different behaviors among the optimizers. FSAM often improves generalization to novel classes but may reduce base-class accuracy. In contrast, SAMPLe maintains a stronger balance between base and novel performance, leading to consistently higher HM across different prompt-learning backbones. These observations suggest that SAMPLe provides a more stable trade-off between specialization on the base classes and transfer to unseen classes. \vspace{-0cm}


\begin{table}[!htbp]
    \vspace{-0.5cm} 
    \centering
    \captionsetup{justification=justified}
    \caption{Comparison of different prompt learning methods on base and new classes. For each method, the last four rows show performance improvements when using the SAM, FSAM, SAGM, and the proposed SAMPLe optimizers.}
    \label{tab:full_camparison}
    \vspace{-0.3cm}

    \begin{tabular}{cccc}
        \vspace{0.1cm}
\begin{minipage}{0.25\textwidth}
    \centering
    \tiny
    (a) \textbf{Average}
    \renewcommand{\arraystretch}{0.99}
    \setlength{\arrayrulewidth}{0.2mm}
    \setlength{\tabcolsep}{1pt}
    \definecolor{lightgray}{gray}{0.9}

    {\fontsize{3pt}{3pt}\selectfont
    \begin{tabular}
    {>{\centering\arraybackslash}m{1.04cm}
     >{\centering\arraybackslash}m{0.55cm}
     >{\centering\arraybackslash}m{0.55cm}|
     >{\centering\arraybackslash}m{0.55cm}
     >{\centering\arraybackslash}m{0.002cm}
    }
    \toprule
    \textbf{Method} & \textbf{Base} & \textbf{New} & \textbf{HM} & \\ 
    \midrule
    CoOp & \textbf{82.69} & 63.22 & 71.66 & \\
    +SAM & 80.10 & 73.28 & 76.07 & \\
    +FSAM & 80.22 & 74.51 & 77.01 & \\
    +SAGM & 81.57 & 74.52 & 77.65 & \\
    \rowcolor{lightgray}
    +SAMPLe & 81.52 & \textbf{76.76} & \textbf{78.88} & \\
    \rowcolor{white}
    \arrayrulecolor{gray}\midrule
    CoCoOp & 80.47 & 71.69 & 75.83 & \\
    +SAM & 80.64 & 73.32 & 76.47 & \\
    +FSAM & 81.15 & 75.21 & 77.86 & \\
    +SAGM & 81.46 & 75.71 & 78.28 & \\
    \rowcolor{lightgray}
    +SAMPLe & \textbf{82.11} & \textbf{77.39} & \textbf{79.51} & \\
    \rowcolor{white}
    \arrayrulecolor{gray}\midrule
    MAPLe & 82.28 & 75.14 & 78.55 & \\
    SAM & 82.99 & 76.60 & 79.42 & \\
    +FSAM & 83.29 & 77.12 & 79.87 & \\
    +SAGM & 83.36 & 77.24 & 79.96 & \\
    \rowcolor{lightgray}
    +SAMPLe & \textbf{83.51} & \textbf{77.68} & \textbf{80.28} & \\
    \rowcolor{white}
    \arrayrulecolor{gray}\midrule
    CoPrompt & 84.00 & 77.23 & 80.48 & \\
    +SAM & 84.61 & 77.98 & 80.99 & \\
    +FSAM & 85.05 & 78.47 & 81.45 & \\
    +SAGM & 85.12 & 78.47 & 81.49 & \\
    \rowcolor{lightgray}
    +SAMPLe & \textbf{85.62} & \textbf{79.03} & \textbf{82.02} & \\
    \rowcolor{white}
    \arrayrulecolor{gray}\midrule
    TCP & 84.13 & 75.36 & 79.51 & \\
    +SAM & 84.17 & 75.58 & 79.64 & \\
    +FSAM & 84.02 & 76.11 & 79.83 & \\
    +SAGM & 86.97 & 76.50 & 81.40 & \\
    \rowcolor{lightgray}
    +SAMPLe & \textbf{87.52} & \textbf{77.05} & \textbf{81.95} & \\

    \bottomrule
    \end{tabular}
    }
\end{minipage}
 &
        
\begin{minipage}{0.25\textwidth}
    \centering
    \tiny
    (b)\textbf{ImageNet}
    \renewcommand{\arraystretch}{0.99}
    \setlength{\arrayrulewidth}{0.2mm}
    \setlength{\tabcolsep}{1pt}
    \definecolor{lightgray}{gray}{0.9}

    {\fontsize{3.0pt}{3pt}\selectfont
    \begin{tabular}
    {>{\centering\arraybackslash}m{1.04cm}
     >{\centering\arraybackslash}m{0.55cm}
     >{\centering\arraybackslash}m{0.55cm}|
     >{\centering\arraybackslash}m{0.55cm}
     >{\centering\arraybackslash}m{0.002cm}
    }
    \toprule
    \textbf{Method} & \textbf{Base} & \textbf{New} & \textbf{HM} & \\ 
    \midrule
    CoOp & \textbf{76.47} & 67.88 & 71.92 & \\
    +SAM & 76.05 & 69.85 & 72.88 & \\
    +FSAM & 69.79 & 70.27 & 70.03 & \\
    +SAGM & 76.40 & 69.91 & 73.01 & \\
    \rowcolor{lightgray}
    +SAMPLe & {76.36} & \textbf{71.10} & \textbf{73.64} & \\
    \rowcolor{white}
    \arrayrulecolor{gray}\midrule
    CoCoOp & \textbf{75.98} & 70.43 & 73.10 & \\
    +SAM & 75.93 & 71.16 & 73.47 & \\
    +FSAM & 75.59 & 71.62 & 73.55 & \\
    +SAGM & 75.69 & 71.77 & 73.68 & \\
    \rowcolor{lightgray}
    +SAMPLe & {75.95} & \textbf{72.51} & \textbf{74.19} & \\
    \rowcolor{white}
    \arrayrulecolor{gray}\midrule
    MAPLe & {76.66} & 70.54 & 73.47 & \\
    +SAM & 76.54 & 71.85 & 74.12 & \\
    +FSAM & 76.49 & 72.25 & 74.31 & \\
    +SAGM & \textbf{76.69} & 72.24 & 74.40 & \\
    \rowcolor{lightgray}
    +SAMPLe & 76.58 & \textbf{73.20} & \textbf{74.85} & \\
    \rowcolor{white}
    \arrayrulecolor{gray}\midrule
    CoPrompt & \textbf{77.67} & 71.27 & 74.33 & \\
    +SAM & 77.23 & 71.98 & 74.51 & \\
    +FSAM & 77.19 & 72.17 & 74.60 & \\
    +SAGM & 77.65 & 72.09 & 74.77 & \\
    \rowcolor{lightgray}
    +SAMPLe & {77.61} & \textbf{72.38} & \textbf{74.90} & \\
    \rowcolor{white}
    \arrayrulecolor{gray}\midrule
    TCP & 77.27 & 69.87 & 73.38 & \\
    +SAM & 77.34 & 69.96 & 73.46 & \\
    +FSAM & 77.11 & 70.13 & 73.47 & \\
    +SAGM & 77.21 & 70.61 & 73.76 & \\
    \rowcolor{lightgray}
    +SAMPLe & \textbf{77.16} & \textbf{70.90} & \textbf{73.90} & \\

    \bottomrule
    \end{tabular}
    }
\end{minipage}
&
        

\begin{minipage}{0.25\textwidth}
    \centering
    \tiny
    (c)\textbf{Caltech101}
    \renewcommand{\arraystretch}{0.99}
    \setlength{\arrayrulewidth}{0.2mm}
    \setlength{\tabcolsep}{1pt}
    \definecolor{lightgray}{gray}{0.9}

    {\fontsize{3pt}{3pt}\selectfont
    \begin{tabular}
    {>{\centering\arraybackslash}m{1.04cm}
     >{\centering\arraybackslash}m{0.55cm}
     >{\centering\arraybackslash}m{0.55cm}|
     >{\centering\arraybackslash}m{0.55cm}
     >{\centering\arraybackslash}m{0.002cm}
    }
    \toprule
    \textbf{Method} & \textbf{Base} & \textbf{New} & \textbf{HM} & \\ 
    \midrule
    CoOp & \textbf{98.00} & 89.81 & 93.73 & \\
    +SAM & 97.05 & \textbf{95.09} & 96.06 & \\
    +FSAM & 96.57 & 93.62 & 95.07 & \\
    +SAGM & 97.16 & 93.78 & 95.44 & \\
    \rowcolor{lightgray}
    +SAMPLe & 97.49 & 94.98 & \textbf{96.22} & \\
    \rowcolor{white}
    \arrayrulecolor{gray}\midrule
    CoCoOP & 97.96 & 93.81 & 95.84 & \\
    +SAM & 97.63 & 94.85 & 96.22 & \\
    +FSAM & 97.54 & 94.90 & 96.20 & \\
    +SAGM & 97.81 & 95.33 & 96.55 & \\
    \rowcolor{lightgray}
    +SAMPLe & \textbf{98.31} & \textbf{96.07} & \textbf{97.18} & \\
    \rowcolor{white}
    \arrayrulecolor{gray}\midrule
    MAPLe & 97.74 & 94.36 & 96.02 & \\
    +SAM & 98.28 & 95.51 & 96.88 & \\
    +FSAM & 98.32 & 95.42 & 96.85 & \\
    +SAGM & \textbf{98.37} & 95.51 & 96.92 & \\
    \rowcolor{lightgray}
    +SAMPLe & 98.31 & \textbf{95.74} & \textbf{97.01} & \\
    \rowcolor{white}
    \arrayrulecolor{gray}\midrule
    CoPrompt & 98.27 & 94.90 & 96.55 & \\
    +SAM & 98.98 & 95.68 & 97.30 & \\
    +FSAM & 98.69 & 95.93 & 97.29 & \\
    +SAGM & 98.72 & 96.12 & 97.40 & \\
    \rowcolor{lightgray}
    +SAMPLe & \textbf{99.03} & \textbf{96.66} & \textbf{97.83} & \\
    \rowcolor{white}
    \arrayrulecolor{gray}\midrule
    TCP & 98.23 & 94.67 & 96.42 & \\
    +SAM & 98.31 & 94.82 & 96.55 & \\
    +FSAM & 98.02 & 94.97 & 96.49 & \\
    +SAGM & 98.62 & 95.83 & 97.20 & \\
    \rowcolor{lightgray}
    +SAMPLe & \textbf{98.94} & \textbf{96.38} & \textbf{97.64} & \\

    \bottomrule
    \end{tabular}
    }
\end{minipage}
&
        

\begin{minipage}{0.25\textwidth}
    \centering
    \tiny
    (d)\textbf{OxfordPets}
    \renewcommand{\arraystretch}{0.99}
    \setlength{\arrayrulewidth}{0.2mm}
    \setlength{\tabcolsep}{1pt}
    \definecolor{lightgray}{gray}{0.9}

    {\fontsize{3pt}{3pt}\selectfont
    \begin{tabular}
    {>{\centering\arraybackslash}m{1.04cm}
     >{\centering\arraybackslash}m{0.55cm}
     >{\centering\arraybackslash}m{0.55cm}|
     >{\centering\arraybackslash}m{0.55cm}
     >{\centering\arraybackslash}m{0.002cm}
    }
    \toprule
    \textbf{Method} & \textbf{Base} & \textbf{New} & \textbf{HM} & \\ 
    \midrule
    CoOp & 93.67 & 95.29 & 94.47 & \\
    +SAM & 93.68 & 97.65 & 95.62 & \\
    +FSAM & 94.22 & 96.38 & 95.29 & \\
    +SAGM & 95.41 & 96.31 & 95.86 & \\
    \rowcolor{lightgray}
    +SAMPLe & \textbf{95.92} & \textbf{98.27} & \textbf{97.08} & \\
    \rowcolor{white}
    \arrayrulecolor{gray}\midrule
    CoCoOp & 95.20 & 97.69 & 96.43 & \\
    +SAM & 94.79 & 97.43 & 96.09 & \\
    +FSAM & 94.96 & 97.58 & 96.25 & \\
    +SAGM & 95.29 & 97.67 & 96.47 & \\
    \rowcolor{lightgray}
    +SAMPLe & \textbf{95.98} & \textbf{98.59} & \textbf{97.27} & \\
    \rowcolor{white}
    \arrayrulecolor{gray}\midrule
    MAPLe & 95.43 & 97.76 & 96.58 & \\
    +SAM & 96.23 & 98.45 & 97.33 & \\
    +FSAM & {96.41} & \textbf{98.63} & 97.51 & \\
    +SAGM & 96.45 & 98.52 & 97.47 & \\
    \rowcolor{lightgray}
    +SAMPLe & \textbf{96.54} & 98.61 & \textbf{97.56} & \\
    \rowcolor{white}
    \arrayrulecolor{gray}\midrule
    CoPrompt & 95.67 & 98.10 & 96.87 & \\
    +SAM & 96.52 & 98.99 & 97.74 & \\
    +FSAM & 96.86 & 99.35 & 98.09 & \\
    +SAGM & 97.00 & 99.25 & 98.11 & \\
    \rowcolor{lightgray}
    +SAMPLe & \textbf{97.65} & \textbf{99.93} & \textbf{98.78} & \\
    \rowcolor{white}
    \arrayrulecolor{gray}\midrule
    TCP & 94.67 & 97.20 & 95.92 & \\
    +SAM & 94.19 & 97.53 & 95.83 & \\
    +FSAM & 94.50 & 98.03 & 96.23 & \\
    +SAGM & 95.33 & 98.18 & 96.73 & \\
    \rowcolor{lightgray}
    +SAMPLe & \textbf{96.58} & \textbf{98.96} & \textbf{97.76} & \\

    \bottomrule
    \end{tabular}
    }
\end{minipage} \\
         \vspace{0.1cm}
\begin{minipage}{0.25\textwidth}
    \centering
    \tiny
    (e)\textbf{StanfordCars}
    \renewcommand{\arraystretch}{0.99}
    \setlength{\arrayrulewidth}{0.2mm}
    \setlength{\tabcolsep}{1pt}
    \definecolor{lightgray}{gray}{0.9}

    {\fontsize{3pt}{3pt}\selectfont
    \begin{tabular}
    {>{\centering\arraybackslash}m{1.04cm}
     >{\centering\arraybackslash}m{0.55cm}
     >{\centering\arraybackslash}m{0.55cm}|
     >{\centering\arraybackslash}m{0.55cm}
     >{\centering\arraybackslash}m{0.002cm}
    }
    \toprule
    \textbf{Method} & \textbf{Base} & \textbf{New} & \textbf{HM} & \\ 
    \midrule
    CoOp & 78.12 & 60.40 & 68.13 & \\
    +SAM & 72.84 & 73.11 & 72.97 & \\
    +FSAM & \textbf{76.13} & 73.92 & \textbf{74.99} & \\
    +SAGM & 74.96 & 73.86 & 74.41 & \\
    \rowcolor{lightgray}
    +SAMPLe & {72.21} & \textbf{75.10} & 73.63 & \\
    \rowcolor{white}
    \arrayrulecolor{gray}\midrule
    CoCoOp & 70.49 & 73.59 & 72.01 & \\
    +SAM & 71.00 & 75.22 & 73.05 & \\
    +FSAM & \textbf{73.92} & 75.80 & 74.85 & \\
    +SAGM & 73.41 & 76.13 & 74.75 & \\
    \rowcolor{lightgray}
    +SAMPLe & 73.27 & \textbf{77.89} & \textbf{75.51} & \\
    \rowcolor{white}
    \arrayrulecolor{gray}\midrule
    MAPLe & 72.94 & 74.00 & 73.47 & \\
    +SAM & 74.21 & 75.83 & 75.01 & \\
    +FSAM & 75.75 & 76.41 & 76.08 & \\
    +SAGM & 75.40 & 76.68 & 76.03 & \\
    \rowcolor{lightgray}
    +SAMPLe & \textbf{76.32} & \textbf{77.44} & \textbf{76.88} & \\
    \rowcolor{white}
    \arrayrulecolor{gray}\midrule
    CoPrompt & 76.97 & 74.40 & 75.66 & \\
    +SAM & 77.64 & 75.00 & 76.30 & \\
    +FSAM & 78.38 & 75.49 & 76.91 & \\
    +SAGM & 78.39 & 75.56 & 76.95 & \\
    \rowcolor{lightgray}
    +SAMPLe & \textbf{79.14} & \textbf{76.03} & \textbf{77.55} & \\
    \rowcolor{white}
    \arrayrulecolor{gray}\midrule
    TCP & 80.80 & 74.13 & 77.32 & \\
    +SAM & 81.03 & 74.13 & 77.43 & \\
    +FSAM & 80.76 & 74.51 & 77.50 & \\
    +SAGM & 82.23 & 75.25 & 78.59 & \\
    \rowcolor{lightgray}
    +SAMPLe & \textbf{83.03} & \textbf{75.70} & \textbf{79.19} & \\

    \bottomrule
    \end{tabular}
    }
\end{minipage}
 &
        
\begin{minipage}{0.25\textwidth}
    \centering
    \tiny
    (f) \textbf{Flowers102}
    \renewcommand{\arraystretch}{0.99}
    \setlength{\arrayrulewidth}{0.2mm}
    \setlength{\tabcolsep}{1pt}
    \definecolor{lightgray}{gray}{0.9}

    {\fontsize{3pt}{3pt}\selectfont
    \begin{tabular}
    {>{\centering\arraybackslash}m{1.04cm}
     >{\centering\arraybackslash}m{0.55cm}
     >{\centering\arraybackslash}m{0.55cm}|
     >{\centering\arraybackslash}m{0.55cm}
     >{\centering\arraybackslash}m{0.002cm}
    }
    \toprule
    \textbf{Method} & \textbf{Base} & \textbf{New} & \textbf{HM} & \\ 
    \midrule
    CoOp & {97.60} & 59.67 & 74.06 & \\
    +SAM & 94.97 & 72.65 & 82.32 & \\
    +FSAM & 94.91 & 73.85 & 83.07 & \\
    +SAGM & 95.25 & 74.53 & 83.63 & \\
    \rowcolor{lightgray}
    +SAMPLe & \textbf{96.73} & \textbf{77.52} & \textbf{86.07} & \\
    \rowcolor{white}
    \arrayrulecolor{gray}\midrule
    CoCoOp & 94.87 & 71.75 & 81.71 & \\
    +SAM & 95.02 & 74.40 & 83.46 & \\
    +FSAM & 95.17 & 76.43 & 84.78 & \\
    +SAGM & 95.50 & 77.01 & 85.26 & \\
    \rowcolor{lightgray}
    +SAMPLe & \textbf{96.47} & \textbf{79.48} & \textbf{87.15} & \\
    \rowcolor{white}
    \arrayrulecolor{gray}\midrule
    MAPLe & 95.92 & 72.46 & 82.56 & \\
    +SAM & 96.53 & 75.03 & 84.43 & \\
    +FSAM & 96.69 & 75.91 & 85.05 & \\
    +SAGM & 96.77 & 76.07 & 85.18 & \\
    \rowcolor{lightgray}
    +SAMPLe & \textbf{96.82} & \textbf{76.98} & \textbf{85.77} & \\
    \rowcolor{white}
    \arrayrulecolor{gray}\midrule
    CoPrompt & 97.27 & 76.60 & 85.71 & \\
    +SAM & 97.79 & 77.55 & 86.50 & \\
    +FSAM & 98.61 & 78.10 & 87.16 & \\
    +SAGM & 98.36 & 78.02 & 87.02 & \\
    \rowcolor{lightgray}
    +SAMPLe & \textbf{99.08} & \textbf{78.47} & \textbf{87.58} & \\
    \rowcolor{white}
    \arrayrulecolor{gray}\midrule
    TCP & 97.73 & 75.57 & 85.23 & \\
    +SAM & 97.99 & 75.37 & 85.21 & \\
    +FSAM & 97.44 & 76.30 & 85.62 & \\
    +SAGM & 98.76 & 76.91 & 86.48 & \\
    \rowcolor{lightgray}
    +SAMPLe & \textbf{99.49} & \textbf{77.36} & \textbf{87.02} & \\

    \bottomrule
    \end{tabular}
    }
\end{minipage}

 &
        
\begin{minipage}{0.25\textwidth}
    \centering
    \tiny
    (g) \textbf{Food101}
    \renewcommand{\arraystretch}{0.99}
    \setlength{\arrayrulewidth}{0.2mm}
    \setlength{\tabcolsep}{1pt}
    \definecolor{lightgray}{gray}{0.9}
    {\fontsize{3pt}{3pt}\selectfont
    \begin{tabular}{>{\centering\arraybackslash}m{1.04cm}
                    >{\centering\arraybackslash}m{0.55cm}
                    >{\centering\arraybackslash}m{0.55cm}|
                    >{\centering\arraybackslash}m{0.55cm}
                    >{\centering\arraybackslash}m{0.002cm}}
    \toprule
    \textbf{Method} & \textbf{Base} & \textbf{New} & \textbf{HM} & \\
    \midrule
    CoOp & 88.33 & 82.26 & 85.19 & \\
    +SAM & 88.19 & 90.12 & 89.14 & \\
    +FSAM & 89.74 & 90.40 & 90.07 & \\
    +SAGM & 90.48 & 91.54 & 91.00 & \\
    \rowcolor{lightgray}
    +SAMPLe & \textbf{90.75} & \textbf{92.01} & \textbf{91.38} & \\
    \rowcolor{white}
    \arrayrulecolor{gray}\midrule
    CoCoOp & 90.70 & 91.29 & 90.99 & \\
    +SAM & 90.83 & 91.70 & 91.26 & \\
    +FSAM & 91.21 & 92.11 & 91.66 & \\
    +SAGM & 91.36 & 92.35 & 91.85 & \\
    \rowcolor{lightgray}
    +SAMPLe & \textbf{91.84} & \textbf{93.08} & \textbf{92.46} & \\
    \rowcolor{white}
    \arrayrulecolor{gray}\midrule
    MAPLe & 90.71 & 92.05 & 91.38 & \\
    +SAM & 91.55 & 92.81 & 92.18 & \\
    +FSAM & 91.88 & 93.20 & 92.54 & \\
    +SAGM & \textbf{91.95} & \textbf{93.35} & \textbf{92.64} & \\
    \rowcolor{lightgray}
    +SAMPLe & 91.94 & 93.31 & 92.62 & \\
    \rowcolor{white}
    \arrayrulecolor{gray}\midrule
    CoPrompt & 90.73 & 92.07 & 91.40 & \\
    +SAM & 91.61 & 92.81 & 92.21 & \\
    +FSAM & 92.10 & 93.47 & 92.78 & \\
    +SAGM & 92.21 & 93.42 & 92.81 & \\
    \rowcolor{lightgray}
    +SAMPLe & \textbf{92.70} & \textbf{94.21} & \textbf{93.45} & \\
    \rowcolor{white}
    \arrayrulecolor{gray}\midrule
    TCP & 90.57 & 91.37 & 90.97 & \\
    +SAM & 90.72 & 91.78 & 91.25 & \\
    +FSAM & 90.55 & 91.92 & 91.23 & \\
    +SAGM & 92.01 & 92.66 & 92.33 & \\
    \rowcolor{lightgray}
    +SAMPLe & \textbf{92.49} & \textbf{93.44} & \textbf{92.96} & \\

    \bottomrule
    \end{tabular}
    }
\end{minipage}
 &
        
\begin{minipage}{0.25\textwidth}
    \centering
    \tiny
    (h) \textbf{FGVCAircrafts}
    \renewcommand{\arraystretch}{0.99}
    \setlength{\arrayrulewidth}{0.2mm}
    \setlength{\tabcolsep}{1pt}
    \definecolor{lightgray}{gray}{0.9}
    {\fontsize{3pt}{3pt}\selectfont
    \begin{tabular}{>{\centering\arraybackslash}m{1.04cm}
                    >{\centering\arraybackslash}m{0.55cm}
                    >{\centering\arraybackslash}m{0.55cm}|
                    >{\centering\arraybackslash}m{0.55cm}
                    >{\centering\arraybackslash}m{0.002cm}}
    \toprule
    \textbf{Method} & \textbf{Base} & \textbf{New} & \textbf{HM} & \\
    \midrule
    CoOp & \textbf{40.44} & 22.30 & 28.75 & \\
    +SAM & 34.33 & 35.87 & 35.08 & \\
    +FSAM & 35.92 & \textbf{39.12} & 37.45 & \\
    +SAGM & 36.73 & 35.75 & 36.23 & \\
    \rowcolor{lightgray}
    +SAMPLe & 36.43 & 38.81 & \textbf{37.58} & \\
    \rowcolor{white}
    \arrayrulecolor{gray}\midrule
    CoCoOp & 33.41 & 23.71 & 27.74 & \\
    +SAM & 34.62 & 29.53 & 31.87 & \\
    +FSAM & 35.45 & 33.67 & 34.54 & \\
    +SAGM & 35.79 & 34.18 & 34.97 & \\
    \rowcolor{lightgray}
    +SAMPLe & \textbf{37.14} & \textbf{36.59} & \textbf{36.86} & \\
    \rowcolor{white}
    \arrayrulecolor{gray}\midrule
    MAPLe & 37.44 & 35.61 & 36.50 & \\
    +SAM & 38.55 & 37.99 & 38.27 & \\
    +FSAM & 39.08 & 38.79 & 38.93 & \\
    +SAGM & 39.11 & 38.37 & 38.74 & \\
    \rowcolor{lightgray}
    +SAMPLe & \textbf{39.47} & \textbf{39.01} & \textbf{39.24} & \\
    \rowcolor{white}
    \arrayrulecolor{gray}\midrule
    CoPrompt & 40.20 & 39.33 & 39.76 & \\
    +SAM & 41.08 & 39.88 & 40.47 & \\
    +FSAM & 41.28 & 40.68 & 40.98 & \\
    +SAGM & 41.45 & 40.68 & 41.06 & \\
    \rowcolor{lightgray}
    +SAMPLe & \textbf{42.38} & \textbf{40.95} & \textbf{41.65} & \\
    \rowcolor{white}
    \arrayrulecolor{gray}\midrule
    TCP & 41.97 & 34.43 & 37.83 & \\
    +SAM & 42.02 & 34.49 & 37.89 & \\
    +FSAM & 42.05 & 35.02 & 38.21 & \\
    +SAGM & 43.23 & 35.55 & 39.02 & \\
    \rowcolor{lightgray}
    +SAMPLe & \textbf{44.19} & \textbf{35.79} & \textbf{39.60} & \\

    \bottomrule
    \end{tabular}
    }
\end{minipage}
 \\
         \vspace{0.1cm}

\begin{minipage}{0.25\textwidth}
    \centering
    \tiny
    (i) \textbf{SUN397}
    \renewcommand{\arraystretch}{0.99}
    \setlength{\arrayrulewidth}{0.2mm}
    \setlength{\tabcolsep}{1pt}
    \definecolor{lightgray}{gray}{0.9}
    {\fontsize{3pt}{3pt}\selectfont
    \begin{tabular}{>{\centering\arraybackslash}m{1.04cm}
                    >{\centering\arraybackslash}m{0.55cm}
                    >{\centering\arraybackslash}m{0.55cm}|
                    >{\centering\arraybackslash}m{0.55cm}
                    >{\centering\arraybackslash}m{0.002cm}}
    \toprule
    \textbf{Method} & \textbf{Base} & \textbf{New} & \textbf{HM} & \\
    \midrule
    CoOp & \textbf{80.60} & 65.89 & 72.51 & \\
    +SAM & 77.84 & 78.82 & 78.33 & \\
    +FSAM & 78.42 & 78.98 & 78.70 & \\
    +SAGM & 78.50 & 79.03 & 78.76 & \\
    \rowcolor{lightgray}
    +SAMPLe & 78.40 & \textbf{79.22} & \textbf{78.81} & \\
    \rowcolor{white}
    \arrayrulecolor{gray}\midrule
    CoCoOp & 79.74 & 76.86 & 78.27 & \\
    +SAM & 79.48 & 77.44 & 78.45 & \\
    +FSAM & 79.73 & 78.35 & 79.03 & \\
    +SAGM & 80.03 & 78.83 & 79.43 & \\
    \rowcolor{lightgray}
    +SAMPLe & \textbf{80.58} & \textbf{79.65} & \textbf{80.11} & \\
    \rowcolor{white}
    \arrayrulecolor{gray}\midrule
    MAPLe & 80.82 & 78.70 & 79.75 & \\
    +SAM & 81.42 & 79.86 & 80.63 & \\
    +FSAM & 81.68 & 80.06 & 80.86 & \\
    +SAGM & 81.74 & 80.26 & 80.99 & \\
    \rowcolor{lightgray}
    MAPLe+SAMPLe & \textbf{81.83} & \textbf{80.48} & \textbf{81.15} & \\
    \rowcolor{white}
    \arrayrulecolor{gray}\midrule
    CoPrompt & 82.63 & 80.03 & 81.31 & \\
    +SAM & 83.27 & 80.73 & 81.98 & \\
    +FSAM & 84.09 & 81.19 & 82.61 & \\
    +SAGM & 83.96 & 81.34 & 82.63 & \\
    \rowcolor{lightgray}
    +SAMPLe & \textbf{84.17} & \textbf{82.08} & \textbf{83.11} & \\
    \rowcolor{white}
    \arrayrulecolor{gray}\midrule
    TCP & 82.63 & 78.20 & 80.35 & \\
    +SAM & 82.83 & 78.66 & 80.70 & \\
    +FSAM & 82.57 & 78.99 & 80.74 & \\
    +SAGM & 83.92 & 79.43 & 81.61 & \\
    \rowcolor{lightgray}
    +SAMPLe & \textbf{84.11} & \textbf{80.16} & \textbf{82.08} & \\

    \bottomrule
    \end{tabular}
    }
\end{minipage}
 &
        
\begin{minipage}{0.25\textwidth}
    \centering
    \tiny
    (j) \textbf{DTD}
    \renewcommand{\arraystretch}{0.99}
    \setlength{\arrayrulewidth}{0.2mm}
    \setlength{\tabcolsep}{1pt}
    \definecolor{lightgray}{gray}{0.9}
    {\fontsize{3pt}{3pt}\selectfont
    \begin{tabular}{>{\centering\arraybackslash}m{1.04cm}
                    >{\centering\arraybackslash}m{0.55cm}
                    >{\centering\arraybackslash}m{0.55cm}|
                    >{\centering\arraybackslash}m{0.55cm}
                    >{\centering\arraybackslash}m{0.002cm}}
    \toprule
    \textbf{Method} & \textbf{Base} & \textbf{New} & \textbf{HM} & \\
    \midrule
    CoOp & \textbf{79.44} & 41.18 & 54.24 & \\
    +SAM & 76.97 & 56.52 & 65.18 & \\
    +FSAM & 73.93 & 54.78 & 62.93 & \\
    +SAGM & 77.31 & 56.76 & 65.46 & \\
    \rowcolor{lightgray}
    +SAMPLe & {78.12} & \textbf{63.04} & \textbf{69.77} & \\
    \rowcolor{white}
    \arrayrulecolor{gray}\midrule
    CoCoOp & 77.01 & 56.00 & 64.85 & \\
    +SAM & 77.05 & 57.13 & 65.61 & \\
    +FSAM & 76.60 & 57.92 & 65.96 & \\
    +SAGM & 77.47 & 59.15 & 67.08 & \\
    \rowcolor{lightgray}
    +SAMPLe & \textbf{78.14} & \textbf{63.05} & \textbf{69.79} & \\
    \rowcolor{white}
    \arrayrulecolor{gray}\midrule
    MAPLe & 80.36 & 59.18 & 68.16 & \\
    +SAM & 80.99 & 61.43 & 69.87 & \\
    +FSAM & 80.88 & 62.55 & 70.54 & \\
    +SAGM & 81.15 & 63.00 & 70.93 & \\
    \rowcolor{lightgray}
    +SAMPLe & \textbf{81.23} & \textbf{63.65} & \textbf{71.37} & \\
    \rowcolor{white}
    \arrayrulecolor{gray}\midrule
    CoPrompt & 83.13 & 64.73 & 72.79 & \\
    +SAM & 83.67 & 65.51 & 73.48 & \\
    +FSAM & 84.49 & 65.74 & 73.94 & \\
    +SAGM & 84.31 & 66.21 & 74.17 & \\
    \rowcolor{lightgray}
    +SAMPLe & \textbf{84.72} & \textbf{66.57} & \textbf{74.56} & \\
    \rowcolor{white}
    \arrayrulecolor{gray}\midrule
    TCP & 82.77 & 58.07 & 68.25 & \\
    +SAM & 82.73 & 58.39 & 68.47 & \\
    +FSAM & 82.84 & 59.41 & 69.17 & \\
    +SAGM & 83.89 & 59.34 & 69.51 & \\
    \rowcolor{lightgray}
    +SAMPLe & \textbf{84.30} & \textbf{59.67} & \textbf{69.88} & \\

    \bottomrule
    \end{tabular}
    }
\end{minipage}
 &
        
\begin{minipage}{0.25\textwidth}
    \centering
    \tiny
    (k) \textbf{EuroSAT}
    \renewcommand{\arraystretch}{0.99}
    \setlength{\arrayrulewidth}{0.2mm}
    \setlength{\tabcolsep}{1pt}
    \definecolor{lightgray}{gray}{0.9}
    {\fontsize{3pt}{3pt}\selectfont
    \begin{tabular}{>{\centering\arraybackslash}m{1.04cm}
                    >{\centering\arraybackslash}m{0.55cm}
                    >{\centering\arraybackslash}m{0.55cm}|
                    >{\centering\arraybackslash}m{0.55cm}
                    >{\centering\arraybackslash}m{0.002cm}}
    \toprule
    \textbf{Method} & \textbf{Base} & \textbf{New} & \textbf{HM} & \\
    \midrule
    CoOp & \textbf{92.19} & 54.74 & 68.69 & \\
    +SAM & 87.00 & 61.23 & 71.87 & \\
    +FSAM & 89.60 & 73.41 & 80.70 & \\
    +SAGM & 91.07 & 72.31 & 80.61 & \\
    \rowcolor{lightgray}
    +SAMPLe & {90.88} & \textbf{75.43} & \textbf{82.44} & \\
    \rowcolor{white}
    \arrayrulecolor{gray}\midrule
    CoCoOp & 87.49 & 60.04 & 71.21 & \\
    +SAM & 88.02 & 62.00 & 72.75 & \\
    +FSAM & 89.19 & 72.15 & 79.77 & \\
    +SAGM & 90.05 & 73.00 & 80.63 & \\
    \rowcolor{lightgray}
    +SAMPLe & \textbf{90.76} & \textbf{75.32} & \textbf{82.32} & \\
    \rowcolor{white}
    \arrayrulecolor{gray}\midrule
    MAPLe & 94.07 & 73.23 & 82.35 & \\
    +SAM & 94.55 & 74.25 & 83.18 & \\
    +FSAM & 94.89 & 75.12 & 83.86 & \\
    +SAGM & 94.92 & 75.38 & 84.03 & \\
    \rowcolor{lightgray}
    +SAMPLe & \textbf{94.96} & \textbf{75.66} & \textbf{84.22} & \\
    \rowcolor{white}
    \arrayrulecolor{gray}\midrule
    CoPrompt & 94.60 & 78.57 & 85.84 & \\
    +SAM & 95.27 & 79.56 & 86.71 & \\
    +FSAM & 95.89 & 80.05 & 87.26 & \\
    +SAGM & 96.03 & 79.85 & 87.20 & \\
    \rowcolor{lightgray}
    +SAMPLe & \textbf{96.75} & \textbf{80.45} & \textbf{87.85} & \\
    \rowcolor{white}
    \arrayrulecolor{gray}\midrule
    TCP & 91.63 & 74.73 & 82.32 & \\
    +SAM & 91.54 & 75.13 & 82.55 & \\
    +FSAM & 91.92 & 75.87 & 83.24 & \\
    +SAGM & 92.96 & 75.90 & 83.57 & \\
    \rowcolor{lightgray}
    +SAMPLe & \textbf{93.66} & \textbf{76.46} & \textbf{84.19} & \\

    \bottomrule
    \end{tabular}
    }
\end{minipage}
 &
        
\begin{minipage}{0.25\textwidth}
    \centering
    \tiny
    (l) \textbf{UCF101}
    \renewcommand{\arraystretch}{0.99}
    \setlength{\arrayrulewidth}{0.2mm}
    \setlength{\tabcolsep}{1pt}
    \definecolor{lightgray}{gray}{0.9}
    {\fontsize{3pt}{3pt}\selectfont
    \begin{tabular}{>{\centering\arraybackslash}m{1.04cm}
                    >{\centering\arraybackslash}m{0.55cm}
                    >{\centering\arraybackslash}m{0.55cm}|
                    >{\centering\arraybackslash}m{0.55cm}
                    >{\centering\arraybackslash}m{0.002cm}}
    \toprule
    \textbf{Method} & \textbf{Base} & \textbf{New} & \textbf{HM} & \\
    \midrule
    CoOp & \textbf{84.69} & 56.05 & 67.46 & \\
    +SAM & 82.16 & 75.13 & 78.49 & \\
    +FSAM & 83.14 & 74.88 & 78.79 & \\
    +SAGM & 83.97 & 75.96 & 79.76 & \\
    \rowcolor{lightgray}
    +SAMPLe & {83.40} & \textbf{78.85} & \textbf{81.06} & \\
    \rowcolor{white}
    \arrayrulecolor{gray}\midrule
    CoCoOp & 82.33 & 73.45 & 77.64 & \\
    +SAM & 82.64 & 75.62 & 78.97 & \\
    +FSAM & 83.26 & 76.77 & 79.88 & \\
    +SAGM & 83.61 & 77.44 & 80.41 & \\
    \rowcolor{lightgray}
    +SAMPLe & \textbf{84.73} & \textbf{79.01} & \textbf{81.77} & \\
    \rowcolor{white}
    \arrayrulecolor{gray}\midrule
    MAPLe & 83.00 & 78.66 & 80.77 & \\
    +SAM & 84.05 & 79.59 & 81.76 & \\
    +FSAM & 84.14 & 80.02 & 82.03 & \\
    +SAGM & 84.36 & 80.24 & 82.25 & \\
    \rowcolor{lightgray}
    +SAMPLe & \textbf{84.62} & \textbf{80.39} & \textbf{82.45} & \\
    \rowcolor{white}
    \arrayrulecolor{gray}\midrule
    CoPrompt & 86.90 & 79.57 & 83.07 & \\
    +SAM & 87.65 & 80.14 & 83.73 & \\
    +FSAM & 88.00 & 81.04 & 84.38 & \\
    +SAGM & 88.27 & 80.67 & 84.30 & \\
    \rowcolor{lightgray}
    +SAMPLe & \textbf{88.55} & \textbf{81.59} & \textbf{84.93} & \\
    \rowcolor{white}
    \arrayrulecolor{gray}\midrule
    TCP & 87.13 & 80.77 & 83.83 & \\
    +SAM & 87.22 & 81.17 & 84.08 & \\
    +FSAM & 87.50 & 82.04 & 84.69 & \\
    +SAGM & 88.45 & 81.83 & 85.01 & \\
    \rowcolor{lightgray}
    +SAMPLe & \textbf{88.73} & \textbf{82.76} & \textbf{85.64} & \\

    \bottomrule
    \end{tabular}
    }
\end{minipage}
\\

    \end{tabular}
\end{table} \vspace{-0cm}

\textbf{{Insights and Ablation Study}:}
Table~\ref{tab:full_camparison} and Fig~\ref{fig:comparison} demonstrate that SAMPLe consistently outperforms existing sharpness-aware optimizers by maintaining both generalization and stability across varying $\rho$ and $\lambda$. While F-SAM mitigates the excessive perturbation of traditional SAM by deviating from the mini-batch gradient, $\nabla^{\mathcal{F}}$, it still enforces a rigid batch-specific gradient, $\nabla^{\mathcal{B}}$, which remains inherently sensitive to the noisy approximation of $\nabla^{\mathcal{F}}$. This structural limitation is evident in Fig~\ref{fig:comparison}, where F-SAM and SAMPLe exhibit comparable robustness to variations in $\rho$, outperforming both SAM and SAGM. However, SAMPLe consistently achieves higher accuracy across all perturbation radii and prompt-learning methods, highlighting its optimization-aware perturbation alignment. Rather than rigidly enforcing batch-specific perturbations, SAMPLe dynamically aligns them with the current optimization state, striking an optimal balance between exploration and stability. This adaptability translates into significant performance gains, as confirmed by Figure~\ref{fig:comparison}. \vspace{-0.2cm}
\subsection{Cross-Dataset Zero-Shot Generalization Setting}
\label{sec:cross_dataset}

To evaluate cross-domain robustness, we train models on the ImageNet dataset containing $1000$ classes and directly evaluate them on $10$ unseen target datasets without any fine-tuning. This protocol measures the ability of prompt learning methods to transfer knowledge learned from a large-scale source domain to diverse downstream domains with varying visual characteristics.

Table~\ref{tab:cross-dataset} summarizes the cross-dataset performance of several prompt learning approaches with and without the proposed SAMPLe optimizer. Overall, incorporating SAMPLe consistently improves the transfer performance of prompt-based models across most datasets. For example, CoOp benefits significantly from SAMPLe, increasing its average accuracy from $63.88\%$ to $65.85\%$. The improvement is particularly notable on datasets such as Cars ($66.35\%$ vs.\ $64.51\%$), Flowers ($71.46\%$ vs.\ $68.71\%$), and Aircraft ($23.61\%$ vs.\ $18.47\%$), which typically exhibit larger domain shifts from ImageNet.

\begin{table*}[!htbp]
    \centering
    \captionsetup{justification=justified}
    \caption{Performance improvement by SAMPLe optimizer in zero-shot cross-dataset.}
    \label{tab:cross-dataset}
    \renewcommand{\arraystretch}{0.75}
    \setlength{\arrayrulewidth}{0.3mm}
    \definecolor{lightgray}{gray}{0.9}
    \begin{tabular}{
        >{\centering\arraybackslash}p{2cm}
        >{\centering\arraybackslash}p{0.75cm}
        >{\centering\arraybackslash}p{0.05cm}
        >{\centering\arraybackslash}p{0.75cm}
        >{\centering\arraybackslash}p{0.75cm}
        >{\centering\arraybackslash}p{0.75cm}
        >{\centering\arraybackslash}p{0.75cm}
        >{\centering\arraybackslash}p{0.75cm}
        >{\centering\arraybackslash}p{0.75cm}
        >{\centering\arraybackslash}p{0.75cm}
        >{\centering\arraybackslash}p{0.75cm}
        >{\centering\arraybackslash}p{0.75cm}
        >{\centering\arraybackslash}p{0.75cm}
        >{\centering\arraybackslash}p{0.5cm}
    }
        \toprule
            \fontsize{6pt}{6pt}\selectfont{
            \textbf{}
            }
        &
            \fontsize{6pt}{6pt}\selectfont{
            \textbf{Source}
            }
        &
            \multicolumn{11}{c}{\fontsize{6pt}{6pt}\selectfont{
            \textbf{Target}
            }}

        \\ 

        \cmidrule{2-2}
        \cmidrule{4-14}

                \fontsize{6pt}{6pt}\selectfont{} & 
                \fontsize{6pt}{6pt}\selectfont{ImNet} &
                \fontsize{6pt}{6pt}\selectfont{} &
                \fontsize{6pt}{6pt}\selectfont{Caltech} & 
                \fontsize{6pt}{6pt}\selectfont{Pets} & 
                \fontsize{6pt}{6pt}\selectfont{Cars} & 
                \fontsize{6pt}{6pt}\selectfont{Flowers} &
                \fontsize{6pt}{6pt}\selectfont{Food} & 
                \fontsize{6pt}{6pt}\selectfont{Aircraft} & 
                \fontsize{6pt}{6pt}\selectfont{SUN} & 
                \fontsize{6pt}{6pt}\selectfont{DTD} &
                \fontsize{6pt}{6pt}\selectfont{EuSAT} & 
                \fontsize{6pt}{6pt}\selectfont{UCF} & 
                \fontsize{6pt}{6pt}\selectfont{Avg} \\
        \midrule
        
                \fontsize{6pt}{6pt}\selectfont{CoOp} & 
                \fontsize{6pt}{6pt}\selectfont{71.51} &
                \fontsize{6pt}{6pt}\selectfont{} &
                \fontsize{6pt}{6pt}\selectfont{93.70} &
                \fontsize{6pt}{6pt}\selectfont{89.14} & 
                \fontsize{6pt}{6pt}\selectfont{64.51} & 
                \fontsize{6pt}{6pt}\selectfont{68.71} & 
                \fontsize{6pt}{6pt}\selectfont{85.30} &
                \fontsize{6pt}{6pt}\selectfont{18.47} & 
                \fontsize{6pt}{6pt}\selectfont{64.15} & 
                \fontsize{6pt}{6pt}\selectfont{41.92} & 
                \fontsize{6pt}{6pt}\selectfont{46.39} &
                \fontsize{6pt}{6pt}\selectfont{66.55} & 
                \fontsize{6pt}{6pt}\selectfont{63.88} \\

            \rowcolor{lightgray}
                \fontsize{6pt}{6pt}\selectfont{+SAMPLe} & 
                \fontsize{6pt}{6pt}\selectfont{70.60} &
                \fontsize{6pt}{6pt}\selectfont{} &
                \fontsize{6pt}{6pt}\selectfont{94.02} & 
                \fontsize{6pt}{6pt}\selectfont{89.90} & 
                \fontsize{6pt}{6pt}\selectfont{66.35} & 
                \fontsize{6pt}{6pt}\selectfont{71.46} &
                \fontsize{6pt}{6pt}\selectfont{86.32} & 
                \fontsize{6pt}{6pt}\selectfont{23.61} & 
                \fontsize{6pt}{6pt}\selectfont{67.61} & 
                \fontsize{6pt}{6pt}\selectfont{44.86} &
                \fontsize{6pt}{6pt}\selectfont{46.74} & 
                \fontsize{6pt}{6pt}\selectfont{67.64} & 
                \fontsize{6pt}{6pt}\selectfont{65.85} \\

        \arrayrulecolor{gray}\midrule
        
                \fontsize{6pt}{6pt}\selectfont{CoCoOp} & 
                \fontsize{6pt}{6pt}\selectfont{71.02} &
                \fontsize{6pt}{6pt}\selectfont{} &
                \fontsize{6pt}{6pt}\selectfont{94.43} &
                \fontsize{6pt}{6pt}\selectfont{90.14} & 
                \fontsize{6pt}{6pt}\selectfont{65.32} & 
                \fontsize{6pt}{6pt}\selectfont{71.88} & 
                \fontsize{6pt}{6pt}\selectfont{86.06} &
                \fontsize{6pt}{6pt}\selectfont{22.94} & 
                \fontsize{6pt}{6pt}\selectfont{67.36} & 
                \fontsize{6pt}{6pt}\selectfont{45.73} & 
                \fontsize{6pt}{6pt}\selectfont{45.37} &
                \fontsize{6pt}{6pt}\selectfont{68.21} & 
                \fontsize{6pt}{6pt}\selectfont{65.74} \\

                \rowcolor{lightgray}
                \fontsize{6pt}{6pt}\selectfont{+SAMPLe} & 
                \fontsize{6pt}{6pt}\selectfont{71.03} &
                \fontsize{6pt}{6pt}\selectfont{{}} &
                \fontsize{6pt}{6pt}\selectfont{{{94.52}}} & 
                \fontsize{6pt}{6pt}\selectfont{{{90.30}}} & 
                \fontsize{6pt}{6pt}\selectfont{{{66.00}}} &
                \fontsize{6pt}{6pt}\selectfont{{{72.84}}} &
                \fontsize{6pt}{6pt}\selectfont{{{86.40}}} & 
                \fontsize{6pt}{6pt}\selectfont{{{23.40}}} &
                \fontsize{6pt}{6pt}\selectfont{{{67.58}}} &
                \fontsize{6pt}{6pt}\selectfont{{{46.75}}} &
                \fontsize{6pt}{6pt}\selectfont{{{46.41}}} & 
                \fontsize{6pt}{6pt}\selectfont{{{69.00}}} & 
                \fontsize{6pt}{6pt}\selectfont{{{66.32}}} \\

        \arrayrulecolor{gray}\midrule       
                \fontsize{6pt}{6pt}\selectfont{MaPLe} & 
                \fontsize{6pt}{6pt}\selectfont{70.72} &
                \fontsize{6pt}{6pt}\selectfont{} &
                \fontsize{6pt}{6pt}\selectfont{93.53} &
                \fontsize{6pt}{6pt}\selectfont{90.49} & 
                \fontsize{6pt}{6pt}\selectfont{65.57} & 
                \fontsize{6pt}{6pt}\selectfont{72.23} & 
                \fontsize{6pt}{6pt}\selectfont{86.20} &
                \fontsize{6pt}{6pt}\selectfont{24.74} & 
                \fontsize{6pt}{6pt}\selectfont{67.01} & 
                \fontsize{6pt}{6pt}\selectfont{46.49} & 
                \fontsize{6pt}{6pt}\selectfont{48.06} &
                \fontsize{6pt}{6pt}\selectfont{68.69} & 
                \fontsize{6pt}{6pt}\selectfont{66.30}\\

                \rowcolor{lightgray}
                \fontsize{6pt}{6pt}\selectfont{+SAMPLe} & 
                \fontsize{6pt}{6pt}\selectfont{70.69} &
                \fontsize{6pt}{6pt}\selectfont{{}} &
                \fontsize{6pt}{6pt}\selectfont{94.65} &
                \fontsize{6pt}{6pt}\selectfont{91.80} &
                \fontsize{6pt}{6pt}\selectfont{66.75} &
                \fontsize{6pt}{6pt}\selectfont{72.59} &
                \fontsize{6pt}{6pt}\selectfont{87.17} &
                \fontsize{6pt}{6pt}\selectfont{23.32} &
                \fontsize{6pt}{6pt}\selectfont{67.83} &
                \fontsize{6pt}{6pt}\selectfont{44.55} &
                \fontsize{6pt}{6pt}\selectfont{53.12} &
                \fontsize{6pt}{6pt}\selectfont{69.58} &
                \fontsize{6pt}{6pt}\selectfont{67.14} \\

        \arrayrulecolor{gray}\midrule

                \fontsize{6pt}{6pt}\selectfont{CoPrompt} & 
                \fontsize{6pt}{6pt}\selectfont{70.80} &
                \fontsize{6pt}{6pt}\selectfont{} &
                \fontsize{6pt}{6pt}\selectfont{94.50} &
                \fontsize{6pt}{6pt}\selectfont{90.73} & 
                \fontsize{6pt}{6pt}\selectfont{65.67} & 
                \fontsize{6pt}{6pt}\selectfont{72.30} & 
                \fontsize{6pt}{6pt}\selectfont{86.43} &
                \fontsize{6pt}{6pt}\selectfont{24.00} & 
                \fontsize{6pt}{6pt}\selectfont{67.57} & 
                \fontsize{6pt}{6pt}\selectfont{47.07} & 
                \fontsize{6pt}{6pt}\selectfont{51.90} &
                \fontsize{6pt}{6pt}\selectfont{69.73} & 
                \fontsize{6pt}{6pt}\selectfont{67.00} \\

                \rowcolor{lightgray}
                \fontsize{6pt}{6pt}\selectfont{+SAMPLe} & 
                \fontsize{6pt}{6pt}\selectfont{70.86} &
                \fontsize{6pt}{6pt}\selectfont{{}} &
                \fontsize{6pt}{6pt}\selectfont{96.24} &
                \fontsize{6pt}{6pt}\selectfont{93.19} &
                \fontsize{6pt}{6pt}\selectfont{67.65} &
                \fontsize{6pt}{6pt}\selectfont{73.59} &
                \fontsize{6pt}{6pt}\selectfont{88.60} &
                \fontsize{6pt}{6pt}\selectfont{21.86} &
                \fontsize{6pt}{6pt}\selectfont{67.75} &
                \fontsize{6pt}{6pt}\selectfont{42.74} &
                \fontsize{6pt}{6pt}\selectfont{47.13} &
                \fontsize{6pt}{6pt}\selectfont{69.71} &
                \fontsize{6pt}{6pt}\selectfont{67.95} \\

        \arrayrulecolor{gray}\midrule

                \fontsize{6pt}{6pt}\selectfont{TCP} & 
                \fontsize{6pt}{6pt}\selectfont{71.40} &
                \fontsize{6pt}{6pt}\selectfont{} &
                \fontsize{6pt}{6pt}\selectfont{93.97} &
                \fontsize{6pt}{6pt}\selectfont{91.25} & 
                \fontsize{6pt}{6pt}\selectfont{64.69} & 
                \fontsize{6pt}{6pt}\selectfont{71.21} & 
                \fontsize{6pt}{6pt}\selectfont{86.69} &
                \fontsize{6pt}{6pt}\selectfont{23.45} & 
                \fontsize{6pt}{6pt}\selectfont{67.15} & 
                \fontsize{6pt}{6pt}\selectfont{44.35} & 
                \fontsize{6pt}{6pt}\selectfont{51.45} &
                \fontsize{6pt}{6pt}\selectfont{68.73} & 
                \fontsize{6pt}{6pt}\selectfont{66.29} \\

                \rowcolor{lightgray}
                \fontsize{6pt}{6pt}\selectfont{+SAMPLe} & 
                \fontsize{6pt}{6pt}\selectfont{71.17} &
                \fontsize{6pt}{6pt}\selectfont{{}} &
                \fontsize{6pt}{6pt}\selectfont{96.82} &
                \fontsize{6pt}{6pt}\selectfont{92.61} &
                \fontsize{6pt}{6pt}\selectfont{67.43} &
                \fontsize{6pt}{6pt}\selectfont{72.17} &
                \fontsize{6pt}{6pt}\selectfont{88.02} &
                \fontsize{6pt}{6pt}\selectfont{23.20} &
                \fontsize{6pt}{6pt}\selectfont{68.20} &
                \fontsize{6pt}{6pt}\selectfont{47.79} &
                \fontsize{6pt}{6pt}\selectfont{47.46} &
                \fontsize{6pt}{6pt}\selectfont{70.61} &
                \fontsize{6pt}{6pt}\selectfont{{{67.53}}} \\

        \bottomrule
    \end{tabular}
\end{table*}

A similar trend can be observed for CoCoOp. While the baseline CoCoOp achieves an average accuracy of $65.74\%$, integrating SAMPLe increases the average performance to $66.32\%$. The improvements are consistently observed across several datasets, including Cars ($66.00\%$ vs.\ $65.32\%$), Flowers ($72.84\%$ vs.\ $71.88\%$), Aircraft ($23.40\%$ vs.\ $22.94\%$), SUN397 ($67.58\%$ vs.\ $67.36\%$), and UCF101 ($69.00\%$ vs.\ $68.21\%$), indicating that SAMPLe improves the robustness of prompt learning under domain shift.
For more advanced prompt learning architectures, SAMPLe continues to provide consistent gains. When applied to MaPLe, the average performance increases from $66.30\%$ to $67.14\%$, with noticeable improvements on datasets such as Caltech ($94.65\%$ vs.\ $93.53\%$), Pets ($91.80\%$ vs.\ $90.49\%$), and EuroSAT ($53.12\%$ vs.\ $48.06\%$). Similarly, TCP benefits from SAMPLe with the average accuracy increasing from $66.29\%$ to $67.53\%$, while achieving higher scores on datasets including Caltech ($96.82\%$ vs.\ $93.97\%$), Cars ($67.43\%$ vs.\ $64.69\%$), SUN397 ($68.20\%$ vs.\ $67.15\%$), and DTD ($47.79\%$ vs.\ $44.35\%$).

Even for strong baselines such as CoPrompt, which already achieves a relatively high average accuracy of $67.00\%$, SAMPLe further improves the overall performance to $67.95\%$. Improvements can be observed across multiple datasets including Caltech ($96.24\%$ vs.\ $94.50\%$), Pets ($93.19\%$ vs.\ $90.73\%$), and Flowers ($73.59\%$ vs.\ $72.30\%$), suggesting that the proposed optimizer provides benefits that are complementary to architectural advances in prompt learning.
Across all evaluated methods, SAMPLe consistently improves the average performance while maintaining stable accuracy on the source domain (ImageNet). 
This behavior indicates that the proposed optimization strategy improves the generalization capability of learned prompts without sacrificing source-domain alignment. 
We attribute these gains to SAMPLe’s ability to explore flatter regions of the loss landscape through its gradient component neutralization mechanism, which helps avoid overfitting to the source distribution and enables more robust transfer to unseen datasets.
\vspace{-0.4cm}
\subsection{Cross-Domain Zero-Shot generalization setting}
\label{sec:cross_domain}
Table~\ref{tab:cross_domain} presents the results for domain generalization. The original ImageNet dataset trains learnable prompts to contextualize the model input in this setup. The evaluation is performed on four diverse ImageNet variants (\textit{-V2}, \textit{-Sketch}, \textit{-Adversarial}, and \textit{-Rendition}), each representing different types of distribution shifts. This evaluation tests how well the model generalizes to unseen domains.
Integrating the proposed SAMPLe optimizer with CoOp and CoCoOp leads to consistent improvements in generalization performance. For example, CoOp+SAMPLe raises the average accuracy from $59.28\%$ (CoOp) to $60.41\%$, with notable gains on datasets such as ImageNet-S (1.32\%), ImageNet-A (1.26\%), and ImageNet-R (1.34\%). Similarly, CoCoOp+SAMPLe increases the average accuracy of CoCoOp from $59.91\%$ to $60.47\%$, showing significant improvements on ImageNet-A (0.48\%) and ImageNet-R (0.56\%).
Compared to other methods like ProGrad, KgCoOp, and MAPLe, CoCoOp+SAMPLe achieves the highest average accuracy of $60.47\%$, surpassing MAPLe ($60.27\%$). This demonstrates SAMPLe’s ability to handle domain shifts more effectively. 

The results show that SAMPLe enhances the model’s exploration during optimization, allowing it to adapt better to challenging datasets such as ImageNet-A and ImageNet-R.
\begin{table*}[!htbp]
    \centering
    \captionsetup{justification=justified}
    \caption{Performance improvement by SAMPLe optimizer in zero-shot domain generalization.}
    \label{tab:cross_domain}
    \renewcommand{\arraystretch}{0.75}
    \setlength{\arrayrulewidth}{0.3mm}
    \definecolor{lightgray}{gray}{0.9}
    \begin{tabular}{
        >{\centering\arraybackslash}p{2cm}
        >{\centering\arraybackslash}p{1.3cm}
        >{\centering\arraybackslash}p{0.05cm}
        >{\centering\arraybackslash}p{1.25cm}
        >{\centering\arraybackslash}p{1.25cm}
        >{\centering\arraybackslash}p{1.25cm}
        >{\centering\arraybackslash}p{1.25cm}
        >{\centering\arraybackslash}p{0.75cm}
    }
        \toprule
            \fontsize{6pt}{6pt}\selectfont{
            \textbf{}
            }
        &
            \fontsize{6pt}{6pt}\selectfont{
            \textbf{Source}
            }
        &
            \multicolumn{5}{c}{\fontsize{6pt}{6pt}\selectfont{
            \textbf{Target}
            }}

        \\ 

        \cmidrule{2-2}
        \cmidrule{4-8}

                \fontsize{6pt}{6pt}\selectfont{} & 
                \fontsize{6pt}{6pt}\selectfont{ImaNet} &
                \fontsize{6pt}{6pt}\selectfont{} &
                \fontsize{6pt}{6pt}\selectfont{ImgNet-V2} & 
                \fontsize{6pt}{6pt}\selectfont{ImgNet-S} & 
                \fontsize{6pt}{6pt}\selectfont{ImgNet-A} & 
                \fontsize{6pt}{6pt}\selectfont{ImgNet-R} &
                \fontsize{6pt}{6pt}\selectfont{Avg} \\
        \midrule
        
                \fontsize{6pt}{6pt}\selectfont{CoOp} & 
                \fontsize{6pt}{6pt}\selectfont{71.51} &
                \fontsize{6pt}{6pt}\selectfont{} &
                \fontsize{6pt}{6pt}\selectfont{64.20} &
                \fontsize{6pt}{6pt}\selectfont{47.99} & 
                \fontsize{6pt}{6pt}\selectfont{49.71} & 
                \fontsize{6pt}{6pt}\selectfont{75.21} & 
                \fontsize{6pt}{6pt}\selectfont{59.28} \\

            \rowcolor{lightgray}
                \fontsize{6pt}{6pt}\selectfont{+SAMPLe} & 
                \fontsize{6pt}{6pt}\selectfont{70.60} &
                \fontsize{6pt}{6pt}\selectfont{} &
                \fontsize{6pt}{6pt}\selectfont{64.43} &
                \fontsize{6pt}{6pt}\selectfont{49.31} & 
                \fontsize{6pt}{6pt}\selectfont{50.97} & 
                \fontsize{6pt}{6pt}\selectfont{{77.34}} & 
                \fontsize{6pt}{6pt}\selectfont{60.41} \\

        \arrayrulecolor{gray}\midrule
        
                \fontsize{6pt}{6pt}\selectfont{CoCoOp} & 
                \fontsize{6pt}{6pt}\selectfont{71.02} &
                \fontsize{6pt}{6pt}\selectfont{} &
                \fontsize{6pt}{6pt}\selectfont{64.07} &
                \fontsize{6pt}{6pt}\selectfont{48.75} & 
                \fontsize{6pt}{6pt}\selectfont{50.63} & 
                \fontsize{6pt}{6pt}\selectfont{76.18} & 
                \fontsize{6pt}{6pt}\selectfont{59.90}  \\

                \rowcolor{lightgray}
                \fontsize{6pt}{6pt}\selectfont{+SAMPLe} & 
                \fontsize{6pt}{6pt}\selectfont{71,03} &
                \fontsize{6pt}{6pt}\selectfont{{}} &
                \fontsize{6pt}{6pt}\selectfont{{{64.31}}} & 
                \fontsize{6pt}{6pt}\selectfont{{{49.00}}} & 
                \fontsize{6pt}{6pt}\selectfont{{{51.05}}} &
                \fontsize{6pt}{6pt}\selectfont{{{77.52}}} &
                \fontsize{6pt}{6pt}\selectfont{{{60.47}}} \\

        \arrayrulecolor{gray}\midrule       
                \fontsize{6pt}{6pt}\selectfont{MaPLe} & 
                \fontsize{6pt}{6pt}\selectfont{70.72} &
                \fontsize{6pt}{6pt}\selectfont{} &
                \fontsize{6pt}{6pt}\selectfont{64.07} &
                \fontsize{6pt}{6pt}\selectfont{49.15} & 
                \fontsize{6pt}{6pt}\selectfont{50.90} & 
                \fontsize{6pt}{6pt}\selectfont{76.98} & 
                \fontsize{6pt}{6pt}\selectfont{60.28} \\

                \rowcolor{lightgray}
                \fontsize{6pt}{6pt}\selectfont{+SAMPLe} & 
                \fontsize{6pt}{6pt}\selectfont{70.69} &
                \fontsize{6pt}{6pt}\selectfont{{}} &
                \fontsize{6pt}{6pt}\selectfont{65.78} &
                \fontsize{6pt}{6pt}\selectfont{50.51} &
                \fontsize{6pt}{6pt}\selectfont{52.11} &
                \fontsize{6pt}{6pt}\selectfont{77.84} &
                \fontsize{6pt}{6pt}\selectfont{{{61.56}}}  \\

        \arrayrulecolor{gray}\midrule

                \fontsize{6pt}{6pt}\selectfont{CoPrompt} & 
                \fontsize{6pt}{6pt}\selectfont{70.80} &
                \fontsize{6pt}{6pt}\selectfont{} &
                \fontsize{6pt}{6pt}\selectfont{64.25} &
                \fontsize{6pt}{6pt}\selectfont{49.43} & 
                \fontsize{6pt}{6pt}\selectfont{50.50} & 
                \fontsize{6pt}{6pt}\selectfont{77.51} & 
                \fontsize{6pt}{6pt}\selectfont{60.42}  \\

                \rowcolor{lightgray}
                \fontsize{6pt}{6pt}\selectfont{+SAMPLe} & 
                \fontsize{6pt}{6pt}\selectfont{70.86} &
                \fontsize{6pt}{6pt}\selectfont{{}} &
                \fontsize{6pt}{6pt}\selectfont{66.12} &
                \fontsize{6pt}{6pt}\selectfont{50.28} &
                \fontsize{6pt}{6pt}\selectfont{51.61} &
                \fontsize{6pt}{6pt}\selectfont{78.50} &
                \fontsize{6pt}{6pt}\selectfont{{{61.63}}} \\

        \arrayrulecolor{gray}\midrule

                \fontsize{6pt}{6pt}\selectfont{TCP} & 
                \fontsize{6pt}{6pt}\selectfont{71.20} &
                \fontsize{6pt}{6pt}\selectfont{{}} &
                \fontsize{6pt}{6pt}\selectfont{{{64.60}}} & 
                \fontsize{6pt}{6pt}\selectfont{{{49.50}}} & 
                \fontsize{6pt}{6pt}\selectfont{{{51.20}}} &
                \fontsize{6pt}{6pt}\selectfont{{{76.73}}} &
                \fontsize{6pt}{6pt}\selectfont{{{60.51}}}  \\

                \rowcolor{lightgray}
                \fontsize{6pt}{6pt}\selectfont{+SAMPLe} & 
                \fontsize{6pt}{6pt}\selectfont{71.17} &
                \fontsize{6pt}{6pt}\selectfont{{}} &
                \fontsize{6pt}{6pt}\selectfont{66.00} &
                \fontsize{6pt}{6pt}\selectfont{50.42} &
                \fontsize{6pt}{6pt}\selectfont{52.19} &
                \fontsize{6pt}{6pt}\selectfont{78.98} &
                \fontsize{6pt}{6pt}\selectfont{{{61.90}}}  \\

        \bottomrule
    \end{tabular}
\end{table*}
\vspace{-0.3cm}

\section{Conclusion}
\label{sec:conclusion}
\vspace{-0.3cm}
This paper introduces SAMPLe, a novel, model-agnostic optimizer that enhances prompt learning across all modalities by integrating sharpness-aware objectives into the optimization process. Inspired by recent advancements in sharpness-aware minimization (SAM), SAMPLe balances exploitation and exploration by considering the optimization state, achieving a harmonious trade-off between accuracy and generalization. A key innovation of SAMPLe lies in its ability to dynamically adapt gradients for prompt learning tasks, ensuring flatter minima that promote robust generalization while maintaining strong performance on seen distributions. Through extensive analysis and rigorous experiments, we demonstrated that SAMPLe significantly improves the generalization capabilities of prompt-based vision-language models (VLMs) and broadly applies to modality-specific prompt-learning tasks. It consistently outperforms counterpart methods across various benchmark datasets, underscoring its effectiveness and robustness in addressing the challenges of prompt learning. Further details are provided in the supplementary material.
\section{Acknowledgments}
\label{sec:conclusion}
This material is based upon work supported by the National Science Foundation under Grant Number CNS-2232048, CNS-2204445, and CCF-2553684.
{
    \small
    \bibliographystyle{splncs04}
    \bibliography{main}
}
\clearpage
\clearpage
\setcounter{page}{1}
\section{Appendix}
In this supplementary material, we begin with the proof of Theorem $1$, followed by a comprehensive comparison of various optimizers, including SAM, FSAM, and SAGM, against the proposed SAMPLe across all $11$ datasets using CoOp, CoCoOp, MaPLe, CoPrompt, and TCP. We then present the ablation study and elaborate on deployment details in the subsequent sections.

\subsection{Proof of Theorem 1}
\label{sec: app_theorm 1}

\begin{lemma}
    Let \( \nabla {L}(\theta_t; \mathcal{D}) \) be the mini-batch gradient at iteration \( t \) and assume that \( \nabla {L}(\theta_t; \mathcal{D}) \) is bounded by \( {L}_{max} \), i.e., \( \nabla {L}(\theta_t; \mathcal{D}) \leq {L}_{max} \). Suppose the exponentially moving average (EMA) of the gradients is given by:
    \begin{equation}
        \begin{aligned}
        m_t = \lambda m_{t-1} + (1 - \lambda) \nabla {L}(\theta_t; \mathcal{D}),
        \end{aligned}
        \label{eq:EMA2}
    \end{equation}
    where \( \lambda \in (0, 1) \). Then the EMA approximation \( m_T \) at iteration \( T \) is bounded by:
    \begin{equation}
        \begin{aligned}
        m_T \leq  \nabla{L}_{max} (1 - \lambda^T).
        \end{aligned}
        \label{eq:upper_bound}
    \end{equation}
    \label{lemma1}
\end{lemma}

\noindent\textit{\textbf{Proof:}}
We proceed by induction and using properties of the geometric series.

At \( t = 1 \), we have:
\begin{equation}
    \begin{aligned}
    m_1 = (1 - \lambda) \nabla {L}(\theta_1; \mathcal{D}),
    \end{aligned}
    \label{eq:m1}
\end{equation}
so,
\begin{equation}
    \begin{aligned}
    m_1 \leq (1 - \lambda) \nabla{L}_{max}.
    \end{aligned}
    \label{eq:m1_bound}
\end{equation}

\noindent For \( t \geq 2 \), the EMA update is given by:
\begin{equation}
    \begin{aligned}
    m_t = \lambda m_{t-1} + (1 - \lambda) \nabla {L}(\theta_{t-1}; \mathcal{D}),
    \end{aligned}
    \label{eq:mt}
\end{equation}
and since \( \nabla {L}(\theta_t; \mathcal{D}) \leq \nabla{L}_{max} \), we can bound:
\begin{equation}
    \begin{aligned}
    m_t \leq (1 - \lambda) \sum_{i=1}^{t} \lambda^{t-i} \nabla{L}_{max}.
    \end{aligned}
    \label{eq:mt_bound}
\end{equation}

\noindent The sum \( \sum_{i=1}^{t} \lambda^{t-i} \) is a geometric series and simplifies to:
\begin{equation}
    \begin{aligned}
    \sum_{i=1}^{t} \lambda^{t-i} = \frac{1 - \lambda^t}{1 - \lambda}.
    \end{aligned}
    \label{eq:geo_series}
\end{equation}

Thus, the upper bound becomes:
\begin{equation}
    \begin{aligned}
    m_T \leq (1 - \lambda) \nabla{L}_{max} \frac{1 - \lambda^T}{1 - \lambda}.
    \end{aligned}
    \label{eq:mT_bound_step}
\end{equation}

Simplifying:
\begin{equation}
    \begin{aligned}
    m_T \leq \nabla{L}_{max} (1 - \lambda^T).
    \end{aligned}
    \label{eq:mT_final}
\end{equation}

Hence, as \( T \to \infty \), \( m_T \) asymptotically approaches \( {L}_{max} \).

\begin{theorem}[definition]
    Assuming the loss function defined as Eq.~ \ref{eq:obj_fn_1}, we have:
    \begin{equation}
        \begin{aligned}
            \mathcal{L}(\theta_t; \mathcal{D}) = & {L}_{}(\theta_{t}; \mathcal{D}) + {L}_{}\big(\theta_{t} + {\epsilon}^{\star}_{t} - \alpha_{t} \nabla {L}(\theta_{t}; \mathcal{D}) ~+\\
       &\alpha_{t} ~ \xi_{t} ~ \sigma_{t} ~ \nabla^{\mathcal{F}} {L}(\theta; \mathcal{D})\big)
        \end{aligned}
    \end{equation}    
    and $\mathcal{L}(\theta_t; \mathcal{D})$ satisfies the following assumptions:
    \begin{itemize}
        \item[(i)] its gradient $\nabla \mathcal{L}(\theta_t; \mathcal{D})$ is bounded, i.e., $\|\nabla \mathcal{L}(\theta_t; \mathcal{D})\| \leq \nabla\mathcal{L}_{max}$, $\forall t$,
        \item[(ii)] The stochastic gradient is K-Lipschitz gradient, i.e., $\|\nabla \mathcal{L}(\theta_t; \mathcal{D}) - \nabla \mathcal{L}(\theta^{'}_{t}; \mathcal{D})\| \leq K \|\theta_t - \theta^{'}_{t}\|$, $\forall \theta_t, \theta^{'}_{t}$.
    \end{itemize}
    Let the learning rate $\eta_t$ be $\eta_t = \frac{\eta_0}{\sqrt{t}}$, and let the perturbations decrease with the same rate as the learning rate, i.e., $\rho_t = \frac{\rho_0}{\sqrt{t}}$ and $\alpha_t = \frac{\alpha_0}{\sqrt{t}}$. If $\hat \theta_{t}$ be defined as follows:
    \begin{equation}
        \begin{aligned}
            &\hat\theta_{t} =  ~\theta_{t} + \epsilon^{*} - \alpha_{t} \sum\limits_{(x,y)\in \mathcal{D}} \nabla {L}(\theta_t, \mathcal{D}) ~+\\ 
            &\alpha_{t} ~ \sum\limits_{(x,y)\in \mathcal{D}} \xi_{t} ~ \sigma_{t} ~  \nabla^{\mathcal{F}} {L}(\theta_{t}; \mathcal{D})\\
            \approx& ~\theta_{t} + \epsilon^{*} - \alpha_{t} \sum\limits_{(x,y)\in \mathcal{D}} \nabla {L}(\theta_t, \mathcal{D}) ~+\alpha_{t} ~ \sum\limits_{(x,y)\in \mathcal{D}} \xi_{t} ~ \sigma_{t} ~  m_{t}
        \end{aligned}
        \label{eq: theta_hat}
    \end{equation}
    where $m_t$ approximates full-gradient as it is defined in Eq.~\ref{eq:EMA}.
    
    \begin{equation}
        \frac{1}{T} \sum_{t=1}^{T}  \left[ \|\nabla \mathcal{L}(\theta_t; \mathcal{D})\|^2 \right] \leq \mathcal{O}\left(\frac{\log T}{\sqrt{T}}\right),
    \end{equation}
    \label{theorem1}
\end{theorem}

\noindent\textit{\textbf{Proof:}}
For simplicity of equations, we define $d_t$ as:
\begin{equation}
    d_{t} = \theta_{t+1} - \theta_{t} = -\eta_{t} \nabla \mathcal{L} (\theta_{t}) - \eta_{t} \nabla \mathcal{L} (\hat \theta_{t})
    \label{eq:d_t}
\end{equation}
where $\eta_t$ represents the learning rate. By K-Lipschitz gradient of $\mathcal{L}$, the definition of $d_t$ in Eq.~\ref{eq:d_t}, and the inequality of $\|\nabla \mathcal{L} (\theta_{t}) + \nabla \mathcal{L} (\hat \theta_{t})\|^2 \leq 2(\|\nabla \mathcal{L} (\theta_{t})\|^2 + \|\nabla \mathcal{L} (\hat\theta_{t})\|^2)$ we need to define an upper bound for loss update difference between every consequent iteration as follows:

\begin{equation}
    \begin{aligned}
        &\mathcal{L}(\theta_{t+1}) - \mathcal{L}(\theta_t) \leq \langle \nabla \mathcal{L}(\theta_t), \theta_{t+1} - \theta_t \rangle + \frac{K}{2} \|\theta_{t+1} - \theta_t\|^2 =\\
        & \langle \nabla \mathcal{L}(\theta_t), d_t \rangle + \frac{K}{2} \|d_t\|^2 = -\eta_t \langle \nabla \mathcal{L}(\theta_t), \nabla \mathcal{L} (\theta_{t}) + \nabla \mathcal{L} (\hat \theta_{t})) \rangle \\
        &\quad \quad + \frac{K \eta_t^2}{2} \|\nabla \mathcal{L} (\theta_{t}) + \nabla \mathcal{L} (\hat \theta_{t})\|^2 =\\
        & -\eta_t \langle \nabla \mathcal{L}(\theta_t), \nabla \mathcal{L}(\theta_t) + \nabla \mathcal{L} (\theta_{t}) -\nabla \mathcal{L}(\theta_t) + \nabla \mathcal{L} (\hat \theta_{t}) \rangle \\
        &\quad\quad  + \frac{K \eta_t^2}{2} \|\nabla \mathcal{L} (\theta_{t}) + \nabla \mathcal{L} (\hat \theta_{t})\|^2 = \\
        & -\eta_t \|\nabla \mathcal{L}(\theta_t)\|^2 - \eta_t \langle \nabla \mathcal{L}(\theta_t), \nabla \mathcal{L}(\theta_t) - \nabla \mathcal{L}(\theta_t) + \nabla \mathcal{L} (\hat \theta_{t}) \rangle \\
        &\quad \quad + \frac{K \eta_t^2}{2} \|\nabla \mathcal{L} (\theta_{t}) + \nabla \mathcal{L} (\hat \theta_{t})\|^2 = \\
        & -2 \eta_t \|\nabla \mathcal{L}(\theta_t)\|^2 - \eta_t \langle \nabla \mathcal{L}(\theta_t), \nabla \mathcal{L} (\hat \theta_{t}) - \nabla \mathcal{L}(\theta_t) \rangle \\
        &\quad \quad + K \eta_t^2 (\|\nabla \mathcal{L} (\theta_{t})\|^2 + \|\nabla \mathcal{L} (\hat\theta_{t})\|^2)\\
        &\leq -2 \eta_t \|\nabla \mathcal{L}(\theta_t)\|^2 + \eta_t \langle \nabla \mathcal{L}(\theta_t), \nabla \mathcal{L} (\theta_{t}) - \nabla \mathcal{L}(\hat\theta_t) \rangle \\
        & \quad \quad + K \eta_t^2 \nabla\mathcal{L}_{max}^2
    \end{aligned}
    \label{eq:prelim}
\end{equation}

\noindent in this step, we should show that there is an upper bound for the expression $\langle \nabla \mathcal{L}(\theta_t), \nabla \mathcal{L} (\hat \theta_{t}) - \nabla \mathcal{L}(\theta_t) \rangle$, it is done as follows;

\begin{equation}
    \begin{aligned}
        &\langle \nabla \mathcal{L}(\theta_t), \nabla \mathcal{L} (\hat \theta_{t}) - \nabla \mathcal{L}(\theta_t) \rangle =\\
        &\leq \|\nabla \mathcal{L}(\theta_t)\|~\|\nabla \mathcal{L} (\hat \theta_{t}) - \nabla \mathcal{L}(\theta_t)\| \\        
        &\leq K ~\|\nabla \mathcal{L}(\theta_t)\|~\|\hat \theta_{t} - \theta_t\| \quad \small{\textit{(K-Lipschitz condition.)}} \\
        & = K~\|\nabla \mathcal{L}(\theta_t)\| \cdot \|\epsilon^{*}_{t} - \alpha_{t} \sum\limits_{(x,y)\in \mathcal{D}} \nabla {L}(\theta_t, \mathcal{D}) +\\ 
        &\quad \quad \alpha_{t} ~ \sum\limits_{(x,y)\in \mathcal{D}} \xi_{t} ~ \sigma_{t} ~  m_{t}\| \quad \small{\textit{(using Eq.~\ref{eq: theta_hat})}} \\
        &\leq K~\|\nabla \mathcal{L}(\theta_t)\| \cdot \|\epsilon^{*}_{t}\| ~+ \\
        & \quad ~K\alpha_{t}~\|\nabla \mathcal{L}(\theta_t)\| \cdot \|\sum\limits_{(x,y)\in \mathcal{D}} (\nabla {L}(\theta_t, \mathcal{D}) - \xi_{t} ~ \sigma_{t}
    ~  m_{t})\| \\
        & \text{\textit{(Lemma.~\ref{lemma1} and the fact that $\nabla L_{max} \leq \nabla \mathcal{L}_{max}$, given}} \\
        &\text{\textit{$\mathcal{D} = \{(x_i, y_i)\}_{i=1}^N$ together would result in:)}}\\
        &\leq K~\rho_{t}~\mathcal\nabla{L}_{max}~+ K ~N~\alpha_{t}~\mathcal\nabla{L}_{max}~\nabla{L}_{max}~ - \\
        & K \alpha_{t}~\nabla\mathcal{L}_{max}~\xi_{t} \leq K(\rho_t+\alpha_{t}) \nabla\mathcal{L}_{max} + ~K~N~\alpha_{t}~\nabla\mathcal{L}_{max}^2
    \end{aligned}
    \label{eq:cos_UpBound}
\end{equation}

\noindent by replacing the upper bound in Eq.~\ref{eq:cos_UpBound} in Eq.~\ref{eq:prelim} we have:
\begin{equation}
    \begin{aligned}
        &\mathcal{L}(\theta_{t+1}) - \mathcal{L}(\theta_t) \leq -2 \eta_t \|\nabla \mathcal{L}(\theta_t)\|^2 ~+\\
        &\quad \quad \eta_t K(\rho_t+\alpha_{t}) \nabla\mathcal{L}_{max} + ~K~N~\eta_t~\alpha_{t}~\nabla\mathcal{L}_{max}^2 +\\
        &\quad \quad K \eta_t^2 \nabla\mathcal{L}_{max}^2
    \end{aligned}
    \label{eq:replace}
\end{equation}

\noindent by rearranging the inequality, it will be as:

\begin{equation}
    \begin{aligned}
        & 2 \sum^{T}_{t=1}\eta_t \|\nabla \mathcal{L}(\theta_t)\|^2 \leq \mathcal{L}(\theta_{t}) - \mathcal{L}(\theta_{t+1}) + K \eta_t^2 \nabla\mathcal{L}_{max}^2~+\\
        &\quad \quad \eta_t K(\rho_t+\alpha_{t}) \nabla\mathcal{L}_{max} + ~K~N~\eta_t~\alpha_{t}~\nabla\mathcal{L}_{max}^2
    \end{aligned}
    \label{eq:rearrange}
\end{equation}

\noindent considering definition of $\eta_{t} = \frac{\eta_{0}}{\sqrt{t}}$ and taking summation over all iterations on the left side of Eq.~\ref{eq:rearrange}, we can define a lower bound as follows:

\begin{equation}
    \begin{aligned}
        2 \frac{\eta_{0}}{\sqrt{T}}\sum^{T}_{t=1} \|\nabla \mathcal{L}(\theta_t)\|^2 \leq 2 \sum^{T}_{t=1}\eta_t \|\nabla \mathcal{L}(\theta_t)\|^2
    \end{aligned}
    \label{eq:L_LB}
\end{equation}

\noindent using telescope sum properties and considering $0 \leq \mathcal{L}_{t} \leq \mathcal{L}_{max} ~\forall t :$
\begin{equation}
    \begin{aligned}
        \sum^{T}_{t=1}\big(\mathcal{L}(\theta_{t}) - \mathcal{L}(\theta_{t+1}) \big) = \mathcal{L}_{1} - \mathcal{L}_{T} \leq \mathcal{L}_{1}
    \end{aligned}
    \label{eq:R_UB}
\end{equation}

\noindent by taking summation over all iterations on the right side of Eq.~\ref{eq:rearrange} and using Eq.~\ref{eq:R_UB} we will have:
\begin{equation}
    \begin{aligned}
        & \sum^{T}_{t=1} \big(\mathcal{L}(\theta_{t}) - \mathcal{L}(\theta_{t+1}) \big) + K \sum^{T}_{t=1} \big(\eta_t(\rho_t+\alpha_{t})\big) \mathcal{L}_{max} ~+\\ 
        & \quad K~N~\sum^{T}_{t=1}(\eta_t~\alpha_{t})~\mathcal{L}_{max}^2 + K \sum^{T}_{t=1}(\eta_t^2) \mathcal{L}_{max}^2 ~\leq \mathcal{L}_{1}~+\\
        & K \sum^{T}_{t=1} \big(\eta_t(\rho_t+\alpha_{t})\big) \nabla\mathcal{L}_{max} + KN\sum^{T}_{t=1}(\eta_t~\alpha_{t})~\nabla\mathcal{L}_{max}^2~+\\
        & \quad \quad K \sum^{T}_{t=1}(\eta_t^2) \nabla\mathcal{L}_{max}^2
    \end{aligned}
    \label{eq:rearrange2}
\end{equation}

\noindent using these lower and upper bounds, we can write the following inequality:
\begin{equation}
    \begin{aligned}
        & 2 \frac{\eta_{0}}{\sqrt{T}}\sum^{T}_{t=1} \|\nabla \mathcal{L}(\theta_t)\|^2 \leq \mathcal{L}_{1}+K \sum^{T}_{t=1} \big(\eta_t(\rho_t+\alpha_{t})\big) \nabla\mathcal{L}_{max} +\\
        & \quad \quad \quad \quad KN\sum^{T}_{t=1}(\eta_t~\alpha_{t})~\nabla\mathcal{L}_{max}^2+ K \sum^{T}_{t=1}(\eta_t^2) \nabla\mathcal{L}_{max}^2~ = \\
        & \quad \quad \quad \quad \mathcal{L}_{1} + K\eta_{0}(\rho_{0}+\alpha_{0}) \nabla\mathcal{L}_{max}\sum^{T}_{t=1} (\frac{1}{t}) +\\ 
        & \quad \quad \quad \quad KN\eta_{0}~\alpha_{0}\sum^{T}_{t=1}(\frac{1}{t})~\nabla\mathcal{L}_{max}^2 + K\eta_{0}^2\sum^{T}_{t=1}(\frac{1}{t})~\nabla\mathcal{L}_{max}^2
    \end{aligned}
    \label{eq:L_LB}
\end{equation}

considering $\sum^{T}_{t=1}(\frac{1}{t}) \leq 1+\log(T)$:

\begin{equation}
    \begin{aligned}
        & \frac{1}{T}\sum^{T}_{t=1} \|\nabla \mathcal{L}(\theta_t)\|^2 \leq \frac{\mathcal{L}_{1}}{2\eta_{0}} + \\
        &\frac{K\mathcal{L}_{max}\big(\rho_{0}+ \eta_{0}\mathcal{L}_{max}+(1+N\mathcal{L}_{max})\alpha_{0}\big)}{2} (\frac{1+\log(T)}{\sqrt{T}})
    \end{aligned}
    \label{eq:L_LB}
\end{equation}

\noindent which means;
    \begin{equation}
        \frac{1}{T} \sum_{t=1}^{T}  \left[ \|\nabla \mathcal{L}(\theta_t; \mathcal{D})\|^2 \right] \leq \mathcal{O}\left(\frac{\log T}{\sqrt{T}}\right),
        \label{eq:convergense}
    \end{equation}
Hence, we conclude that the proposed loss function \( \mathcal{L}(\theta_t; \mathcal{D}) \) converges with a rate comparable to first-order gradient-based optimization methods such as Adam and RMSProp, provided that the model is trained for sufficiently large iterations \( T \).
\begin{table}[h]
    \centering
    \captionsetup{justification=justified}
    \caption{Details of Datasets.}
    \label{tab:datasets_stat}
    \renewcommand{\arraystretch}{0.8}
    \setlength{\arrayrulewidth}{0.3mm}
    \definecolor{lightgray}{gray}{0.9}
    \begin{tabular}{
        >{\centering\arraybackslash}p{2cm}
        >{\centering\arraybackslash}p{1.5cm}
        >{\centering\arraybackslash}p{1.5cm}
        >{\centering\arraybackslash}p{1.5cm}
    }
        \toprule




                \fontsize{8pt}{8pt}\selectfont{\textbf{Dataset}} & 
                \fontsize{8pt}{8pt}\selectfont{\textbf{Number of Class}} & \fontsize{8pt}{8pt}\selectfont{\textbf{Train Samples}} & \fontsize{8pt}{8pt}\selectfont{\textbf{Test Samples}}  \\
        \midrule
        
                \fontsize{8pt}{8pt}\selectfont{ImageNet} &
                \fontsize{8pt}{8pt}\selectfont{1000} &
                \fontsize{8pt}{8pt}\selectfont{1.28 M} & 
                \fontsize{8pt}{8pt}\selectfont{50000}  \\

        \midrule
        
                \fontsize{8pt}{8pt}\selectfont{Caltech101} &
                \fontsize{8pt}{8pt}\selectfont{100} &
                \fontsize{8pt}{8pt}\selectfont{4128} & 
                \fontsize{8pt}{8pt}\selectfont{2465} \\

        \midrule

                \fontsize{8pt}{8pt}\selectfont{OxfordPets} &
                \fontsize{8pt}{8pt}\selectfont{37} & 
                \fontsize{8pt}{8pt}\selectfont{2944} & 
                \fontsize{8pt}{8pt}\selectfont{3669}  \\

        \midrule

                \fontsize{8pt}{8pt}\selectfont{StanfordCars} &
                \fontsize{8pt}{8pt}\selectfont{196} & 
                \fontsize{8pt}{8pt}\selectfont{6509} & 
                \fontsize{8pt}{8pt}\selectfont{8041}  \\

        \midrule
        
                \fontsize{8pt}{8pt}\selectfont{Flower102} &
                \fontsize{8pt}{8pt}\selectfont{102} &
                \fontsize{8pt}{8pt}\selectfont{4093} & 
                \fontsize{8pt}{8pt}\selectfont{2463}  \\

        \midrule
        
                \fontsize{8pt}{8pt}\selectfont{Food101} &
                \fontsize{8pt}{8pt}\selectfont{101} &
                \fontsize{8pt}{8pt}\selectfont{50500} & 
                \fontsize{8pt}{8pt}\selectfont{30300} \\
        \midrule

                \fontsize{8pt}{8pt}\selectfont{FGVCAircraft} &
                \fontsize{8pt}{8pt}\selectfont{100} & 
                \fontsize{8pt}{8pt}\selectfont{3334} & 
                \fontsize{8pt}{8pt}\selectfont{3333}  \\

        \midrule

                \fontsize{8pt}{8pt}\selectfont{SUN397} &
                \fontsize{8pt}{8pt}\selectfont{397} & 
                \fontsize{8pt}{8pt}\selectfont{15880} & 
                \fontsize{8pt}{8pt}\selectfont{19850}  \\

        \midrule
        
                \fontsize{8pt}{8pt}\selectfont{DTD} &
                \fontsize{8pt}{8pt}\selectfont{47} &
                \fontsize{8pt}{8pt}\selectfont{2820} & 
                \fontsize{8pt}{8pt}\selectfont{1692}  \\

        \midrule
        
                \fontsize{8pt}{8pt}\selectfont{EuroSSAT} &
                \fontsize{8pt}{8pt}\selectfont{10} &
                \fontsize{8pt}{8pt}\selectfont{13500} & 
                \fontsize{8pt}{8pt}\selectfont{8100} \\
        \midrule
        
                \fontsize{8pt}{8pt}\selectfont{UCF101} &
                \fontsize{8pt}{8pt}\selectfont{101} &
                \fontsize{8pt}{8pt}\selectfont{7639} & 
                \fontsize{8pt}{8pt}\selectfont{3783} \\

        \arrayrulecolor{gray}\midrule
            \rowcolor{lightgray}
                \fontsize{8pt}{8pt}\selectfont{ImageNet-V2} &
                \fontsize{8pt}{8pt}\selectfont{1000} & 
                \fontsize{8pt}{8pt}\selectfont{N/A} & 
                \fontsize{8pt}{8pt}\selectfont{10000}\\

        \arrayrulecolor{gray}\midrule
            \rowcolor{lightgray}
                \fontsize{8pt}{8pt}\selectfont{ImageNet-Sketch} &
                \fontsize{8pt}{8pt}\selectfont{1000} & 
                \fontsize{8pt}{8pt}\selectfont{N/A} & 
                \fontsize{8pt}{8pt}\selectfont{50889}\\

        \arrayrulecolor{gray}\midrule
            \rowcolor{lightgray}
                \fontsize{8pt}{8pt}\selectfont{ImageNet-A} &
                \fontsize{8pt}{8pt}\selectfont{200} & 
                \fontsize{8pt}{8pt}\selectfont{N/A} & 
                \fontsize{8pt}{8pt}\selectfont{7500}\\

        \arrayrulecolor{gray}\midrule
            \rowcolor{lightgray}
                \fontsize{8pt}{8pt}\selectfont{ImageNet-R} &
                \fontsize{8pt}{8pt}\selectfont{200} & 
                \fontsize{8pt}{8pt}\selectfont{N/A} & 
                \fontsize{8pt}{8pt}\selectfont{30000}\\
                
        \bottomrule
    \end{tabular}
\end{table}
\subsection{Optimizer Impact on CoOp, CoCoOp, MaPLe, and CoPrompt}
\label{full_camparison}
Table~\ref{tab:full_camparison} provides a comprehensive comparison of different prompt learning methods enhanced with various optimizers, including SAM, FSAM, SAGM, and the proposed SAMPLe, evaluated across all 11 datasets.

\subsection{Datasets Details}
\label{ssec:datatsets_Appndx}

We describe the $11$ datasets and $4$ variations of ImageNet used for evaluation, providing details on the number of classes, along with the training and testing sample sizes, as summarized in Table.~\ref{tab:datasets_stat}.

\subsection{Robustness to Perturbation Radius}
Unlike SAM and SAGM, which maximize the worst-case perturbation per mini-batch, SAMPLe and F-SAM mitigate instability by constraining perturbations to align with the mini-batch gradient rather than strictly enforcing adversarial maximization. While F-SAM explicitly computes a batch-specific gradient perturbation, SAMPLe encourages alignment with the mini-batch gradient direction without rigidly restricting updates. This distinction allows SAMPLe to balance batch-specific adaptation and generalization, enhancing stability across perturbation scales while achieving superior accuracy without excessive sensitivity to $\rho$ in Fig.~\ref{fig:comparison}.

    

\subsection{Robustness to Whole-Batch Gradient Approximation}
Figure~\ref{fig:comparison} demonstrates that SAMPLe is significantly more robust to variations in $\lambda$, which controls the effect of previous gradients in approximating the full-batch gradient. Unlike F-SAM, which strictly enforces the batch-specific gradient in its perturbation update, SAMPLe only encourages alignment, allowing for greater adaptability. This flexibility enables SAMPLe to dynamically compensate for changes in $\lambda$ based on the optimization state, leading to consistently higher accuracy and lower variance.

\begin{figure}[t]
    \centering
    \begin{minipage}{0.25\textwidth}
        \centering
        \includegraphics[width=\linewidth]{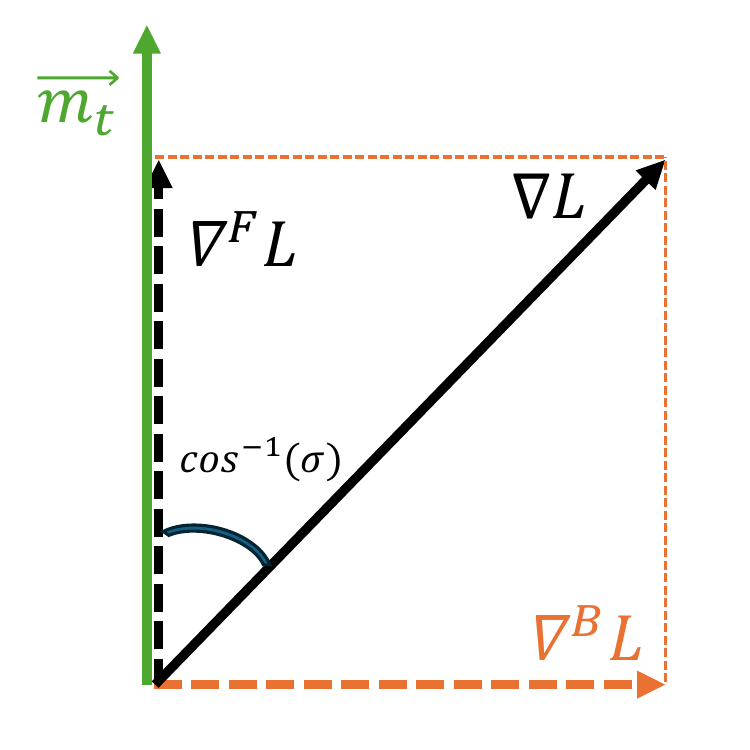}
    \end{minipage}

    \caption{Gradient decomposition .}
    \label{fig:projection}
\end{figure}

\subsection{Deployment Details}
\label{deployment}
\subsubsection{Hyperparameter selection}

\begin{table}[h]
    \centering
    \captionsetup{justification=justified}
    \caption{Hyperparameter settings for SAM, SAGM, and SAMPLe applied to CoOp, CoCoOp, MaPLe, Co-Prompt, and TCP during training on all 11 datasets (in Table.~\ref{tab:full_camparison}).}
    \label{tab:Base2New_hyperParam}
    \renewcommand{\arraystretch}{0.95}
    \setlength{\arrayrulewidth}{0.2mm}
    \definecolor{lightgray}{gray}{0.9}
    \begin{tabular}{
        >{\arraybackslash}p{2.10cm}|
        >{\centering\arraybackslash}p{0.95cm}|
        >{\centering\arraybackslash}p{0.75cm}
        >{\centering\arraybackslash}p{0.95cm}|
        >{\centering\arraybackslash}p{0.85cm}
        >{\centering\arraybackslash}p{0.75cm}|
        >{\centering\arraybackslash}p{0.75cm}
        >{\centering\arraybackslash}p{0.75cm}
        >{\centering\arraybackslash}p{0.75cm}
    }
        \toprule

    {\fontsize{6}{7}\selectfont {\rule{0pt}{0ex}}}&
    {\fontsize{8}{7}\selectfont { \rule{-5pt}{3ex}\textbf{SAM}}}&
    \multicolumn{2}{c|}{{\fontsize{8}{7}\selectfont { \rule{0pt}{0ex}\textbf{SAGM}}}}&
    \multicolumn{2}{c|}{{\fontsize{8}{7}\selectfont { \rule{0pt}{0ex}\textbf{F-SAM}}}}&
    \multicolumn{3}{c}{{\fontsize{8}{7}\selectfont { \rule{-5pt}{3ex}\textbf{SAMPLe}}}}\\ [1.5ex]
    \cline{2-2}
    \cline{3-4}
    \cline{5-6}
    \cline{6-9}
    {\fontsize{6}{7}\selectfont {\rule{0pt}{0ex}}}&
    {\fontsize{8}{7}\selectfont { \rule{-5pt}{3ex}$\mathbf{\mathcal \rho}$}}&
    {\fontsize{8}{7}\selectfont { \rule{-5pt}{3ex}$\mathbf{\mathcal \rho}$}}&
    {\fontsize{8}{7}\selectfont { \rule{-5pt}{3ex}$\mathbf{\mathcal \alpha}$}}&
    {\fontsize{8}{7}\selectfont { \rule{-5pt}{3ex}$\mathbf{\mathcal \rho}$}}&
    {\fontsize{8}{7}\selectfont { \rule{-5pt}{3ex}$\mathbf{\mathcal \lambda}$}}&
    {\fontsize{8}{7}\selectfont { \rule{-5pt}{3ex}$\mathbf{\mathcal \rho}$}}&
    {\fontsize{8}{7}\selectfont { \rule{-5pt}{3ex}$\mathbf{\mathcal \alpha}$}}&
    {\fontsize{8}{7}\selectfont { \rule{-5pt}{3ex}$\mathbf{\mathcal \lambda}$}}\\
        \midrule
        
                \fontsize{7pt}{7pt}\selectfont{CoOp} &
                \fontsize{7pt}{7pt}\selectfont{0.05} &
                \fontsize{7pt}{7pt}\selectfont{0.05} & 
                \fontsize{7pt}{7pt}\selectfont{0.0010} &
                \fontsize{7pt}{7pt}\selectfont{0.05} &
                \fontsize{7pt}{7pt}\selectfont{0.15} &
                \fontsize{7pt}{7pt}\selectfont{0.05} &
                \fontsize{7pt}{7pt}\selectfont{0.0015} &
                \fontsize{7pt}{7pt}\selectfont{0.15}\\
                \fontsize{7pt}{7pt}\selectfont{CoCoOp} &
                \fontsize{7pt}{7pt}\selectfont{0.10} &
                \fontsize{7pt}{7pt}\selectfont{0.10} & 
                \fontsize{7pt}{7pt}\selectfont{0.0010} &
                \fontsize{7pt}{7pt}\selectfont{0.05} &
                \fontsize{7pt}{7pt}\selectfont{0.15} &
                \fontsize{7pt}{7pt}\selectfont{0.10} &
                \fontsize{7pt}{7pt}\selectfont{0.0015} &
                \fontsize{7pt}{7pt}\selectfont{0.15}\\
                \fontsize{7pt}{7pt}\selectfont{MaPLe} &
                \fontsize{7pt}{7pt}\selectfont{0.05} &
                \fontsize{7pt}{7pt}\selectfont{0.10} & 
                \fontsize{7pt}{7pt}\selectfont{0.0010} &
                \fontsize{7pt}{7pt}\selectfont{0.05} &
                \fontsize{7pt}{7pt}\selectfont{0.15} &
                \fontsize{7pt}{7pt}\selectfont{0.10} &
                \fontsize{7pt}{7pt}\selectfont{0.0015} &
                \fontsize{7pt}{7pt}\selectfont{0.15}\\
                \fontsize{7pt}{7pt}\selectfont{CoPrompt} &
                \fontsize{7pt}{7pt}\selectfont{0.05} &
                \fontsize{7pt}{7pt}\selectfont{0.05} & 
                \fontsize{7pt}{7pt}\selectfont{0.0010} &
                \fontsize{7pt}{7pt}\selectfont{0.05} &
                \fontsize{7pt}{7pt}\selectfont{0.15} &
                \fontsize{7pt}{7pt}\selectfont{0.05} &
                \fontsize{7pt}{7pt}\selectfont{0.0015} &
                \fontsize{7pt}{7pt}\selectfont{0.15}\\
                \fontsize{7pt}{7pt}\selectfont{TCP} &
                \fontsize{7pt}{7pt}\selectfont{0.05} &
                \fontsize{7pt}{7pt}\selectfont{0.10} & 
                \fontsize{7pt}{7pt}\selectfont{0.0010} &
                \fontsize{7pt}{7pt}\selectfont{0.05} &
                \fontsize{7pt}{7pt}\selectfont{0.15} &
                \fontsize{7pt}{7pt}\selectfont{0.10} &
                \fontsize{7pt}{7pt}\selectfont{0.0015} &
                \fontsize{7pt}{7pt}\selectfont{0.15}\\
                
        \bottomrule
    \end{tabular}
\end{table}

Table~\ref{tab:Base2New_hyperParam} presents the tuned hyperparameters for SAM, SAGM, and SAMPLe applied to CoOp, CoCoOp, MaPLe, Co-Prompt, and TCP across 11 datasets, focusing on the perturbation radius, $\rho$, alignment parameter, $\alpha$, and previous gradient effective factor, $\lambda$. 
The results in Table~\ref{tab:full_camparison} are well aligned with the hyperparameter choices in Table~\ref{tab:Base2New_hyperParam}. In particular, SAMPLe preserves the same perturbation radius used by the sharpness-aware baselines (\(\rho=0.05\) or \(0.10\), depending on the prompt learner), retains the regularization strength of F-SAM (\(\lambda=0.15\)), and adopts a slightly stronger correction factor than SAGM (\(\alpha=0.0015\) vs.\ \(0.0010\)). This combination yields a more favorable balance between optimization stability and generalization, which is reflected in the results: across all five prompt learning methods, SAMPLe achieves the best average harmonic mean and, in most cases, the strongest new-class accuracy, while maintaining competitive or improved base-class performance. The gains are especially pronounced on datasets with larger base-to-new generalization gaps, such as Flowers102, DTD, EuroSAT, StanfordCars, UCF101, and FGVCAircraft, where SAMPLe consistently improves the balance between memorization of base classes and transfer to unseen categories. Importantly, these improvements are obtained with fixed hyperparameter settings for each method across all 11 datasets, indicating that the performance gains are not due to dataset-specific tuning, but rather to the optimizer itself. Overall, the evidence suggests that SAMPLe more effectively controls loss sharpness and reduces over-specialization, leading to stronger and more consistent base-to-new generalization than SAM, F-SAM, and SAGM.

\subsubsection{MaPLe}
We train MaPLe using a staged training strategy to improve generalization and stability. The training begins with optimizing the learnable prompts and the projection layers using standard optimization methods such as SGD or Adam. Once the initial training phase converges, we freeze the projection layers and continue training only the learnable prompts using SAMPLe, which is designed for prompt learning. This step ensures that the prompts adapt to both flatter minima and low values of loss, which enhances generalization while preserving the alignment between vision and text representations. By decoupling the optimization of prompts and projections, our method allows each component to specialize without disrupting the overall model convergence. The results in Table.~\ref{tab:full_camparison} demonstrates that the SAMPLe optimizer performs well in staged strategy, leading to significantly better performance than the benchmark and other optimization methods, including SAM, F-SAM, and SAGM. Specifically, our approach achieves higher accuracy across the HM value, 

\subsubsection{Co-Prompt}
Following the staged training strategy introduced in MaPLe, we apply SAMPLe to CoPrompt by first training the model using standard optimization and jointly updating the learnable prompts, projection layers, and adapter networks. After convergence, we freeze the adapters and projection layers and then optimize only the learnable prompts using SAMPLe. This allows them to refine their representations while preserving the alignment learned in earlier stages. This staged optimization ensures that the learned prompts remain discriminative and generalize well across unseen class distributions.
\begin{table}[t]
    \vspace{-0cm}
    \centering
    \captionsetup{justification=justified}
    \caption{term-wise ablation study of objective function (Eq.~\eqref{eq:aligned_sam_objective}) average performance over all 11 datasets using CoPrompt method.}
    \label{term wise ablation}
    \renewcommand{\arraystretch}{0.9}
    \setlength{\tabcolsep}{3pt}
    \scriptsize   
    \begin{tabular}{lcc|c}
        \toprule
        \quad\quad\quad\quad\quad\quad\quad\quad\quad\quad\quad\quad & \textbf{Base}\quad & \textbf{New}\quad & \textbf{HM} \\
        \midrule
        $L$ & 84.00 & 77.23 & 80.48 \\
        $L_p$ & {84.61} & 77.98 & 80.99 \\
        $L + L_p - \alpha(\nabla L \cdot \nabla L_p)$ & 85.12 & 78.47 & 81.49 \\
        $L + L_p - \alpha(\nabla L \cdot \nabla L_p) + \alpha \xi \sigma(\nabla L_p \cdot \nabla^{\mathcal{F}} L)$ & 85.62 & 79.03 & 82.02 \\
        \bottomrule
    \end{tabular}
\end{table}

\section{Term-wise ablation study}
To assess the effectiveness of the objective function presented in Eq.~\ref{eq:aligned_sam_objective}, we perform an ablation study against each term. In particular, we compare ERM, $L(\theta, \mathcal{D})$, SAM, ${L}_{p}(\theta; \mathcal{D})$,  SAGM, ${L}(\theta; \mathcal{D}) + {L}_{p}(\theta; \mathcal{D}) - 
\alpha \big(\nabla_{} {L}_{p}(\theta; \mathcal{D}) \cdot \nabla_{} {L}_{}(\theta; \mathcal{D})\big)$ , and SAMPLe, that consists of all terms. All items outperform ERM in the HM column, underscoring the effectiveness of SAM-based methods for prompt learning. Moreover, the exploration term in SAMPLe consistently improves performance across Base, New, and HM, highlighting the importance of balanced learning that accounts for both exploitation and exploration in prompt learning. 

\end{document}